\documentclass[conference]{IEEEtran}
\IEEEoverridecommandlockouts

\usepackage{cite}
\usepackage{amsmath,amssymb,amsfonts}
\usepackage{algorithmic}
\usepackage{graphicx}
\usepackage{textcomp}
\usepackage{xcolor}
\def\BibTeX{{\rm B\kern-.05em{\sc i\kern-.025em b}\kern-.08em
    T\kern-.1667em\lower.7ex\hbox{E}\kern-.125emX}}

\newcount\Comments  
\usepackage{color}
\definecolor{darkgreen}{rgb}{0,0.5,0}
\definecolor{purple}{rgb}{1,0,1}
\newcommand{\kibitz}[2]{\ifnum\Comments=1\textcolor{#1}{#2}\fi}

\newcommand{\robab}[1]{\kibitz{blue}{[Robab: #1]}}

\newcommand{\del}[1]{\kibitz{red}{[Deleted: #1]}}

\usepackage{comment}

\usepackage{amsthm}
\theoremstyle{definition}
\newtheorem{definition}{Definition}

\theoremstyle{definition}

\theoremstyle{plain}
\newtheorem{theorem}{Theorem}

\theoremstyle{remark}
\newtheorem{remark}{Remark}

\usepackage{amsmath}
\usepackage{amsthm}
\usepackage{bbm}

\usepackage[utf8]{inputenc}
\usepackage{times}
\usepackage{soul}
\usepackage{url}
\usepackage[hidelinks]{hyperref}
\usepackage[utf8]{inputenc}
\usepackage{caption}
\usepackage{booktabs}
\usepackage{algorithm}

\usepackage{nomencl}
\makenomenclature

\usepackage{hyperref,color}    
\usepackage[hypcap]{caption}  

\usepackage{subcaption}

\usepackage{xcolor}
\usepackage{tikz}
\usepackage{fancyvrb}
\usepackage{alltt}
\usepackage{multirow}

\usepackage{makecell}

\makeatletter
\newcommand{\linebreakand}{%
  \end{@IEEEauthorhalign}
  \hfill\mbox{}\par
  \mbox{}\hfill\begin{@IEEEauthorhalign}
}
\makeatother

\begin{document}

\title{Modeling Memory-Dependent Reliability of LLMs: A Hidden Markov Model\\
}

\author{
\IEEEauthorblockN{
Robab~Aghazadeh~Chakherlou\IEEEauthorrefmark{1},
Siddartha~Khastgir\IEEEauthorrefmark{1},
Peter~Popov\IEEEauthorrefmark{2}, and
Xingyu~Zhao\IEEEauthorrefmark{1} \thanks{Corresponding author: X.~Zhao at \texttt{xingyu.zhao@warwick.ac.uk}}
}

\IEEEauthorblockA{
\IEEEauthorrefmark{1}\textit{WMG, University of Warwick},
Coventry, United Kingdom
\\
\{robab.aghazadeh-chakherlou,s.khastgir.1,xingyu.zhao\}@warwick.ac.uk
}

\IEEEauthorblockA{
\IEEEauthorrefmark{2}\textit{Centre for Software Reliability, City St George's, University of London},
London, United Kingdom
\\
p.t.popov@citystgeorges.ac.uk
}

\IEEEauthorblockA{
}
}

\maketitle

\begin{abstract}
Reliability assessment of large language models (LLMs) seeks to estimate the probability that a model produces correct responses under a specified operational profile. Conventional benchmark-based evaluation, often summarized by aggregate accuracy, provides a point estimate of performance but does not characterize the uncertainty associated with reliability claims. Currently, statistical inference methods for LLM reliability assessment are emerging. However, a key assumption underlying these models is that test outcomes can be treated as independent repeated trials. This assumption may be inappropriate in sequential settings, where later responses depend on earlier interactions through retained context, error propagation, or an evolving interaction state. 

We extend a hierarchical Bayesian framework for LLM reliability assessment by relaxing the assumption of independent task outcomes and introducing a Hidden Markov Model to capture sequential dependence in benchmark-constructed interaction sessions. In this formulation, outcomes are generated from a latent interaction state evolving according to a first-order Markov process, capturing changes in interaction context.  Through experiments using Anthropic Claude and OpenAI on four datasets, we demonstrate the potential impact of sequential dependence on reliability assessment. The results suggest that ignoring sequential dependence may lead to overconfident reliability estimates.

\end{abstract}


\IEEEpeerreviewmaketitle

\section{Introduction}
\label{sec_introduction}

Large language models (LLMs) are increasingly deployed in
interactive settings where users engage with the model through
sequences of tasks\footnote{A task is an evaluated interaction with the LLM — consisting of one or more related prompts — that produces a binary outcome (success/failure). In the benchmark settings used here, each task comprises a single prompt.} within a continuous interaction, rather
than as isolated evaluations in which each task is performed
independently. In such settings, the operational profile\footnote{A probability distribution over tasks (or subdomains) representing how tasks are encountered in practice.} (OP) provides a description of real-world usage by specifying how frequently different tasks are encountered in practice. In sequential interaction settings, however, performance may also depend on the order in which tasks are encountered, as outcomes can be influenced by the evolving interaction context. This raises the question of whether task outcomes can be treated as independent, as assumed in recent studies.

Early software reliability research shows that failures are not uniformly distributed over the input space, but occur under particular input conditions~\cite{bishop_pods_1988}.
Building on this, failure-inducing inputs may form contiguous regions in the input space (“failure regions” or “blobs”), leading to clustered failures along certain trajectories~\cite{bishop_variation_1993}. The statistical implications of such dependence show that assuming independence can lead to misleading reliability claims when failures cluster over time or along operational trajectories~\cite{strigini_testing_1996}.
This suggests a broader principle: failures tend to cluster when underlying conditions evolve. While the input space of LLMs differs from that of classical
software systems, an analogous effect may arise in sequential\footnote{Throughout this paper, we use sequential to refer to the ordered structure of task execution and temporal to refer to the statistical dependence between outcomes arising from that sequential structure.}
interaction settings, where a sequence of tasks is performed
within a continuous session and context is retained across
tasks. In particular, accumulated context, error propagation,
or reasoning drift may create interaction conditions under
which failures become more likely, leading to temporal
clustering of failures along an interaction trajectory.
As a result, task outcomes may exhibit temporal dependence,
challenging the common assumption of independence. Ignoring temporal dependence may lead to overconfident reliability estimates by failing to capture failure clustering and thus underestimating uncertainty.  

This challenge is not unique to sequential interaction: even in standard benchmark evaluation, Miller~\cite{miller_adding_2024} shows that grouped questions lead to clustered observations and underestimated uncertainty when independence is assumed. However, the dependence considered there is structural within datasets rather than temporal across interactions — the form of dependence this work addresses.

The independence assumption is reasonable when tasks are executed in isolation, such as in offline benchmarking where each task is evaluated in a fresh session, so outcomes can be modeled as independent Bernoulli trials. However, many real-world uses of LLMs involve sequential interactions rather than isolated tasks. Users may engage in multi-step reasoning, long coding sessions, or extended dialogues where context accumulates. In such settings, errors may propagate, reasoning may degrade, and interaction dynamics can affect subsequent performance.
Consequently, the independence assumption may not adequately represent reliability under realistic deployment. One alternative is to define reliability at the level
of an entire interaction session, for example as the probability
that a session is successful. However, this does not capture how
performance evolves during interaction or how failures arise. Instead, we model dependence directly by extending the hierarchical Bayesian framework in~\cite{luettgau_hibayes_2025, aghazadeh_hierarchical_2025}, which represents reliability across the LLM–domain–subdomain hierarchy (Fig.~\ref{fig_hierarchical structure}).

Several modeling approaches can represent temporal dependence 
in sequential task outcomes. Autoregressive logistic models 
condition success probabilities on previous observed outcomes, 
capturing dependence through the history of failures and 
successes directly. Bayesian change-point 
models identify structural breaks in performance, assuming 
constant behavior within segments and abrupt transitions between them.
A Markov model offers a different perspective: rather than conditioning on observed outcomes or assuming abrupt shifts, it posits that the system occupies a state at each time step, and that this state evolves gradually according to transition probabilities. This is appropriate for the LLM 
interaction setting, where performance may vary due to 
unobserved factors that persist and evolve over successive 
tasks rather than changing abruptly or being fully explained 
by past outcomes alone. Under a first-order Markov 
assumption,\footnote{A first-order model is adopted here for 
parsimony; higher-order dependencies can in principle be 
captured by expanding the state space, at the cost of 
increased computational complexity and parameter estimation 
demands.} the effect of the entire past history on future 
performance is captured by the current state,\footnote{In a 
Markov model, the state represents the current condition of 
the system and summarizes past information relevant for 
predicting future outcomes.} so that the next outcome depends 
only on the present state rather than the full history of 
previous tasks. Table~\ref{tab_model_comparison} summarizes these 
alternatives. 

In the LMM interaction setting, however, the interaction condition is not 
directly observable from task outcomes. We therefore adopt a 
Hidden Markov Model (HMM), which extends the Markov framework 
by treating this condition as a latent variable inferred from 
the observed outcome sequence. In an HMM, a latent state 
evolves according to a Markov process and the observed 
performance outcomes are generated conditionally on this 
state, separating the dynamics of the unobserved interaction 
condition from the observable performance record. This 
interpretation is supported by work showing that LLMs can 
implicitly perform HMM-like inference~\cite{dai_pre_2025} 
and by HMM-based modeling of interaction 
dynamics~\cite{lim_evaluating_2025}.

\begin{table*}[h!]
\centering
\caption{Alternative approaches for modeling temporal dependence.}
\label{tab_model_comparison}
\begin{tabular}{p{3.9cm}p{4.4cm}p{8.5cm}}
\hline
\textbf{Model} &
\textbf{Dependence representation} &
\textbf{Primary focus} \\
\hline
Autoregressive logistic model~\cite{kaufmann_regression_1987} &
Dependence on lagged observed outcomes &
Models temporal dependence through previous successes and failures, but does not represent unobserved interaction conditions. \\
Observed-outcome Markov model &
Markov dependence on previous observed outcomes &
Captures short-term dependence but does not distinguish latent interaction regimes from observed successes and failures. \\
Bayesian change-point model~\cite{barry_autoreg_1993} &
Piecewise-constant latent regimes separated by change points &
Designed to detect abrupt structural changes in performance rather than model continuously evolving interaction states. \\
Hidden Markov model &
Latent-state temporal dynamics &
Represents dependence through unobserved interaction states that evolve over time and influence future task outcomes. \\
\hline
\end{tabular}
\end{table*}

\del{Temporal dependence is commonly modeled using state-based stochastic processes (e.g., Markov models) which capture the effect of past observations through a latent state.
Recent studies suggest that LLM generation can be viewed as a sequential stochastic process with Markov-like structure \cite{zekri_large_2024, ildiz_self_2024}, supporting this perspective.}

\del{Under a first-order\footnote{A first-order model is adopted here for parsimony; higher-order dependencies can in principle be captured by expanding the state space, at the cost of increased computational complexity and parameter estimation demands.} Markov assumption, the effect of the entire past
history on future performance is captured by the current state\footnote{In
a Markov model, the state represents the current condition of the system and summarizes past information relevant for predicting future outcomes.}. Formally, the next outcome depends only on the present
state rather than the full history of previous tasks. 
Thus, the probability of success on the next task depends
on the current interaction condition, and this condition evolves over time
in a stochastic but structured manner.
}

\del{In the context of LLM sequential use, this interaction state is not directly observable. Factors such as internal reasoning mode, context accumulation, or interaction fatigue may influence performance, but they cannot be directly measured. Instead, we only observe task outcomes (e.g., success/failure).}

\del{To account for this unobserved interaction condition, we adopt a Hidden Markov Model (HMM). In an HMM, a latent ``state'' evolves according to a Markov process, and the observed performance outcomes are generated conditionally on this latent state. 
The latent state represents the current interaction condition — summarizing all unobserved factors from past interactions, such as dialogue coherence, accumulated context distortions from earlier errors, or fluctuations in reasoning stability, that may affect future performance. This interpretation is supported by work showing that LLMs can implicitly perform HMM-like inference~\cite{dai_pre_2025} and by HMM-based modeling of interaction dynamics~\cite{lim_evaluating_2025}.}

\del{Given the hierarchical structure of the LLM (Fig.~\ref{fig_hierarchical structure}), interaction dynamics can be modeled in several ways. A global-state formulation uses a single latent state across domains, while the domain-specific formulation adopted here uses separate states per domain, capturing within-domain dependence only (even within a single session\footnote{Although tasks may occur within a single interaction session, the sequential dependence is modeled separately within each domain, with independent latent state processes. This simplifies the model and preserves interpretability.}). 
Hybrid formulations capturing both global and local dynamics are theoretically possible but require more data and more complex inference.
Having specified the sequential dynamics within each domain, we next account for how tasks are distributed across domains and subdomains in practice through the OP.}

Given the hierarchical structure of the LLM (Fig.~\ref{fig_hierarchical structure}), interaction dynamics can be modeled in several ways. A global-state formulation uses a single latent state across domains, whereas the domain-specific formulation adopted here uses separate states per domain, capturing only within-domain dependence (even within a single session\footnote{Although tasks may occur within a single interaction session, the sequential dependence is modeled separately within each domain, with independent latent state processes. This simplifies the model and preserves interpretability.}). Hybrid formulations combining global and local dynamics are possible but require more data and more complex inference. 

Following established principles from software reliability, we incorporate the OP as a structured description of real-world usage. Within our hierarchical setting (Fig.~\ref{fig_hierarchical structure}), following~\cite{aghazadeh_hierarchical_2025}, the OP specifies how frequently different domains and subdomains are encountered during deployment by assigning a weight to each domain and subdomain. These operational weights allow reliability to be aggregated from subdomains to domains and ultimately to the LLM level in a usage-aware manner. Thus, reliability is not treated as a simple average over benchmark datasets, but as a deployment-relevant measure that reflects the mix of tasks the model is expected to face.
In this work, the OP is treated as a static distribution over task types: it captures relative frequencies, but not the temporal structure of task arrivals. This provides a tractable baseline for incorporating usage information into the reliability analysis, while the implications of this assumption are discussed later.

\del{The main contributions are:
(i) identifying sequential dependence as a key challenge in LLM reliability assessment, and showing that ignoring it leads to overconfident reliability estimates;
(ii) introducing a hierarchical Hidden Markov Model that extends the state-of-the-art hierarchical Bayesian framework to capture temporal dependence between task outcomes within a continuous interaction session; and
(iii) deriving the posterior distributions of reliability measures under the proposed hierarchical HMM.}
The main contributions are: (i) identifying sequential dependence as a challenge in LLM reliability assessment and demonstrating its impact on reliability estimates; (ii) extending a hierarchical Bayesian reliability framework with a Hidden Markov Model to capture temporal dependence in continuous interaction sessions; and (iii) deriving the corresponding posterior reliability distributions.
\robab{It could be be better to say:  identifying sequential dependence as a challenge in LLM reliability assessment, relaxing the standard independence assumption to better reflect realistic interaction settings, and demonstrating its impact on reliability estimates. I may use this for camera ready version.}

The remainder of this paper is organized as follows: Section~\ref{sec_related_work} reviews related work; Section~\ref{sec_reliability_HMM} introduces the HMM; Section~\ref{sec_numerical_example} presents a numerical example; and Section~\ref{sec_discussion_conclusion} concludes.

\section{Related work}
\label{sec_related_work}

Violating the assumptions underlying reliability assessment can lead to
overly optimistic reliability claims~\cite{littlewood_conservative_1997}.
One of the most important assumptions is independence between test
outcomes, which can be difficult to justify in practice, particularly
in sequential settings where earlier executions may influence later ones.
In software reliability, this challenge has been approached in two ways:
i) by acknowledging the uncertainty about independence without
explicitly modeling it; 2) by directly modeling the dependence structure. The first approach avoids committing to independence by incorporating
uncertainty about this assumption directly into the analysis. For example,
a conservative Bayesian inference derive reliability bounds that remain valid
even when dependence is present~\cite{salako_unnecessity_2023}. The second approach models dependence explicitly using stochastic processes \del{Markov-based frameworks have been widely used in this context:}(e.g., Markov-based frameworks): a binary Markov model introduced in~\cite{chen_binary_1996} shows that dependence can significantly affect reliability estimates, and a Markov renewal framework developed in~\cite{goseva_popstojanova_failure_2000} demonstrates that correlated executions lead to failure clustering. When the state driving dependence is not directly observable, Hidden Markov Models (HMMs) provide a natural extension by introducing a latent state, and have been applied in areas such as degradation~\cite{gamiz_hidden_2023} and resilience assessment~\cite{liu_resilience_2025}.
\del{degradation modeling~\cite{gamiz_hidden_2023} and system resilience assessment~\cite{liu_resilience_2025}.}

An analogous challenge arises in LLM reliability assessment. In sequential interaction settings, the outcome of a task may be influenced by previous interactions, leading to dependence across executions — mirroring the failure clustering observed in software reliability. Sensitivity analysis further shows that ignoring dependence between outcomes can materially affect reliability estimates~\cite{aghazadeh_hierarchical_2025}.
This motivates explicit modeling of sequential dependence using
state-based frameworks. Supporting this, the token generation
mechanism of LLMs has been formally characterized as a
finite-state Markov chain~\cite{zekri_large_2024}, and
transformer dynamics have been related to Markov-type
stochastic processes~\cite{ildiz_self_2024}, suggesting that
LLM outputs have inherent sequential structure. HMMs have also
been applied to sequential language analysis, demonstrating
the value of modeling sequential dynamics in language
settings~\cite{lim_evaluating_2025}. Together, these results support modeling sequential task outcomes as dependent and motivate the use of HMMs for LLM reliability assessment. \del{Together, these results
support the view that task outcomes in sequential interactions
may not be independent, and that HMMs provide a principled
framework for capturing this dependence in LLM reliability
assessment.}

Uncertainty quantification is central to reliability assessment.
In LLMs, hierarchical Bayesian models quantify uncertainty in
reliability across domains~\cite{luettgau_hibayes_2025, aghazadeh_hierarchical_2025},
while uncertainty at the level of responses has also
been studied~\cite{geng_survey_2024, liu_uncertainty_2025}.
However, no existing work combines temporal dependence modeling
with Bayesian inference. We address this gap by
extending the hierarchical framework with an HMM.

\section{LLM Reliability under Temporal Dependence}
\label{sec_reliability_HMM}

\subsection{Domain-Specific Hidden Markov Model}
\label{sec_markov_model}
\del{We consider the reliability assessment of an LLM under three levels of hierarchy: 
the overall model (LLM), domains, and subdomains. Let $D_i$ ($D^{(i)}$ in Fig.~\ref{fig_hierarchical structure}) denote the $i$-th domain (e.g., reasoning, coding), and $S_{ij}$ the $j$-th subdomain within $D_i$ (e.g., specific datasets or task categories).}
We assess LLM reliability across three hierarchical levels: the overall model (LLM), domains, and subdomains. Let $D_i$ ($D^{(i)}$ in Fig.~\ref{fig_hierarchical structure}) denote the $i$-th domain (e.g., reasoning or coding), and $S_{ij}$ the $j$-th subdomain within $D_i$ (e.g., a dataset or task category).
\begin{figure}[!htbp]
    \centering
    \includegraphics[width=0.99\linewidth]{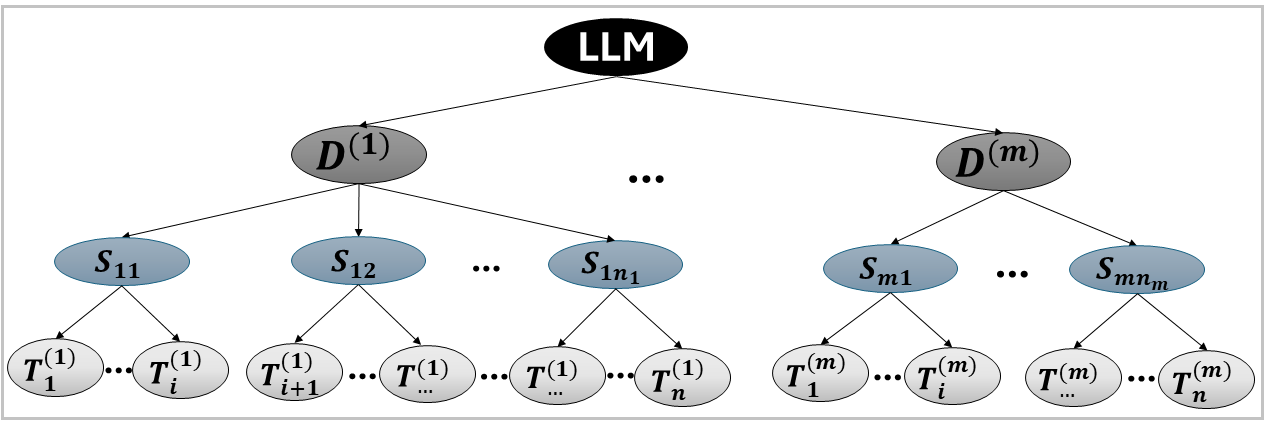}
    \caption{Hierarchical structure of the LLM evaluation. The model is divided into domains and subdomains, with tasks occurring sequentially within each domain.}
    \label{fig_hierarchical structure}
\end{figure}
Each subdomain is associated with a probability of success ($\theta_{ij}$) 
capturing the combined effect of task difficulty and model reliability 
within that subdomain. Rather than treating these probabilities as 
independent, they are linked through shared domain-level hyperparameters 
that govern the mean reliability across subdomains and the degree to which 
individual subdomains may deviate from it (Fig.~\ref{fig_partial_pooling}).
This induces partial pooling, whereby each subdomain’s reliability is informed not only by its own data but also by information from other subdomains within the same domain. \del{This is implemented via a hierarchical Bayesian formulation.}
\begin{figure}[!htbp]
    \centering
    \includegraphics[width=0.55\linewidth]{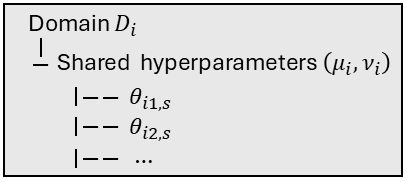}
    \caption{Hierarchical partial pooling in domain \(D_i\): subdomains share hyperparameters ($\mu_i, \nu_i$), so each is informed by its own data and others in the domain.}
    \label{fig_partial_pooling}
\end{figure}

While the hierarchical structure captures cross-subdomain dependence 
through partial pooling, it assumes that task outcomes are independent 
over time. To capture temporal dependence between successive tasks, we 
extend this framework using a HMM.\del{, which models the evolution of an 
unobserved interaction state over time.} In this formulation, a latent 
interaction state evolves according to a Markov process and influences 
the success probabilities of all subdomains within the domain, allowing 
the model to capture how reliability changes during sequential 
interactions. As a result, two distinct forms of dependence are 
incorporated within the same hierarchical structure: cross-subdomain 
dependence induced by partial pooling, and temporal dependence induced 
by the latent state dynamics of the HMM. Here, we focus on temporal dependence, as partial pooling across subdomains is well studied~\cite{luettgau_hibayes_2025, aghazadeh_hierarchical_2025}.

Given the schematic representation of the HMM for domain $D_i$ in Fig.~\ref{fig_hmm_structure}, we now describe the model formally. Parameters shown in the figure (e.g.,  $\theta_{ij,s}$, $A_i$, etc.) are defined as the HMM is developed throughout this section.

\begin{figure}[!htbp]
    \centering
    \includegraphics[width=0.99\linewidth]{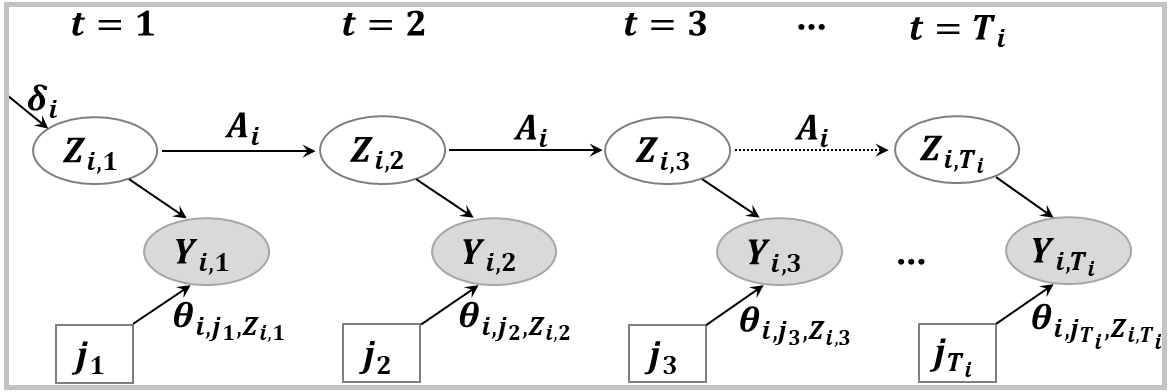}
    \caption{\small{Hidden Markov Model for domain $D_i$. The interaction starts in a latent state $Z_{i,1} \sim \delta_i$ (uniform). At each time $t$, the subdomain label $j_t$ is observed and the outcome $Y_{i,t}$ is generated with success probability $\theta_{i,j_t,Z_{i,t}}$ ($\theta_{ij,s}$ with $j=j_t,\; s=Z_{i,t}$). The state then evolves to $Z_{i,t+1}$ via transition matrix $A_i$. This process repeats, so outcomes depend on both the latent state and task type ($\text{initial state} \rightarrow \text{observe } j_t \rightarrow \text{generate } Y_{i,t} \rightarrow \text{transition} \rightarrow \text{repeat}$). Circles denote random variables ($Z_{i,t}$, $Y_{i,t}$), shaded nodes observed outcomes, and rectangles observed subdomain labels $j_t$.}}
    \label{fig_hmm_structure}
\end{figure}

\paragraph{Session structure and observed variables}

We consider a single continuous interaction session\footnote{A session is a sequence of tasks in a continuous interaction within domain $D_i$, where context is retained, tasks may span subdomains, and a shared latent state evolves over time.} within domain $D_i$, 
during which tasks are executed sequentially without resetting the model's 
context. 
Let $t \in \{1, \dots, T_i\}$ denote the chronological position of a task 
within that interaction session, where $T_i$ is the total number of tasks 
performed sequentially in domain $D_i$. \del{At each task position $t$, we 
observe a subdomain label $j_t \in \{1, \dots, n_i\}$, where $n_i$ denotes the number of subdomains in domain $D_i$, and a binary outcome 
$Y_{i,t} \in \{0,1\}$, where $Y_{i,t} = 1$ denotes success and $Y_{i,t} 
= 0$ denotes failure,} At each task position $t$, we observe a subdomain label $j_t \in \{1,\dots,n_i\}$, where $n_i$ denotes the number of subdomains in domain $D_i$, and a binary outcome $Y_{i,t}\in\{0,1\}$, where $Y_{i,t}=1$ denotes success and $Y_{i,t}=0$ denotes failure. The subdomain label may vary across tasks, allowing successive tasks within a session to belong to different subdomains.
\del{Successive tasks within a session may belong to different subdomains: the subdomain label $j_t$ can vary freely across time steps $t = 1, \ldots, T_i$.} 
\del{The OP is not used to model
the observed within-session ordering, but later to weight task
types in the reliability aggregation.}
The ordering of tasks $(j_t, Y_{i,t})$ across different subdomains within a domain reflects the operational sequence in which they are encountered. The effect of task ordering is discussed in Remark~\ref{remark_order_of_tasks} and analyzed in Sec.~\ref{sec_numerical_example}.

\paragraph{Latent interaction state dynamics}

While the observed sequence $(j_t, Y_{i,t})$ records task types and outcomes, it does not capture the underlying interaction condition that may evolve during the session. In doing this, we introduce a latent interaction state
process $\{Z_{i,t}\}_{t=1}^{T_i}$, where $Z_{i,t} \in \{1, \dots, K\}$
denotes the latent state at time $t$. We use $s$ to represent a particular state in this set ($Z_{i,t}=s$), and $K$
is the total number of latent interaction states. The
value of $K$ is selected in a data-driven manner using
predictive model comparison based on the out-of-sample log
predictive density, a strictly proper scoring rule
that rewards models for assigning higher probability to the
observed outcomes and therefore provides a principled
criterion for model selection~\cite{gneiting_strictly_2007}
(Remark~\ref{remark_number_of_states}). The latent states are not directly observable and do not have a predefined physical interpretation. They provide an abstract representation of the interaction condition governing task performance. Any interpretation must therefore be inferred indirectly from their estimated emission and transition characteristics (Remark~\ref{remark_interpreting_states}).

The interaction state is defined at the \emph{domain level} and affects all 
subdomains within $D_i$. It represents a sufficient summary of the recent interaction
history within $D_i$ and evolves according to a
first-order Markov chain.

Because the latent state is unobserved, the state at the start of a session is uncertain and is therefore represented by a probability distribution over the $K$ possible states. In Fig.~\ref{fig_hmm_structure} we denote by $\delta_i = (\delta_{i,1}, \dots, \delta_{i,K})$ the initial state distribution over the $K$ latent states, where $\delta_{i,s} = p(Z_{i,1} = s)$ and $\sum_{s=1}^{K} \delta_{i,s} = 1$. The latent state then evolves as
\[
Z_{i,1} \sim \delta_i, \qquad Z_{i,t+1} \mid Z_{i,t} \sim A_i,
\]
where $A_i$ is a row-stochastic transition matrix whose entries $(A_i)_{s,s'}$ represent the probability of transitioning from state $s$ to state $s'$ (Fig.~\ref{fig_hmm_structure}), where $s$ and $s'$ denote the realized values of $Z_{i,t}$ and $Z_{i,t+1}$, respectively.
\[
(A_i)_{s,s'} = p(Z_{i,t+1}=s' \mid Z_{i,t}=s).
\]
\del{We explicitly assume a first-order Markov property for the latent state sequence, which states that
$$p(Z_{i,t+1} \mid Z_{i,t}, Z_{i,t-1}, \ldots) = p(Z_{i,t+1} \mid Z_{i,t}).$$}

We assume that the transition matrix $A_i$ is constant
over time. This assumption is discussed in more detail in
Remark~\ref{remark_constant_matrix}.

The initial distribution is fixed as uniform over latent states, 
and for the transition matrix $A_i$, we place Dirichlet priors on 
its rows (see Remark~\ref{remark_delta_Ai}). Given the observed 
sequence $\{(j_t, Y_{i,t})\}_{t=1}^{T_i}$, the joint posterior 
distribution is
\[
p\!\left( A_i, Z_{i,1:T_i} \mid Y_{i,1:T_i}, j_{1:T_i} \right).
\]
which is used to infer the latent-state dynamics and to compute posterior reliability distributions for future tasks.

\paragraph{State-dependent emission model}

\del{For each domain $D_i$, subdomain $j$, and latent interaction
state $s \in \{1,\dots,K\}$, we introduce a parameter
$\theta_{ij,s} \in (0,1)$ representing the probability of task
success when the interaction is in state $s$ and the task
belongs to subdomain $j$. Conditional on the latent state and
the subdomain label, task outcomes follow
\[
Y_{i,t} \mid (Z_{i,t}=s,\, j_t=j) \sim \mathrm{Bernoulli}(\theta_{ij,s}).
\]}
For each domain $D_i$, subdomain $j$, and latent state
$s\in\{1,\dots,K\}$, let $\theta_{ij,s}\in(0,1)$ denote the
probability of task success when the interaction is in state
$s$ and the task belongs to subdomain $j$. Conditional on the
latent state and subdomain label,
\[
Y_{i,t}\mid(Z_{i,t}=s,\,j_t=j)
\sim \mathrm{Bernoulli}(\theta_{ij,s}).
\]

\paragraph{Hierarchical prior structure}
\del{To estimate the success probabilities $\theta_{ij,s}$ across subdomains and interaction states while preserving the partial-pooling structure (Figs.~\ref{fig_hierarchical structure},~\ref{fig_partial_pooling}), we place a common domain-level prior on $\theta_{ij,s}$. That is, all $\theta_{ij,s}$ corresponding to different subdomains $j$ and interaction states $s$ within the same domain $D_i$ share the same hyperparameters $(\mu_i, \nu_i)$, allowing information to be pooled across both subdomains and interaction states while still allowing each $(j,s)$ pair to have its own reliability level.
Specifically, for each $D_i$, we assume: $\mu_i \sim \mathrm{Beta}(a_i, b_i)$, $\nu_i \sim \mathrm{Gamma}(c_i, \mathrm{rate}=d_i)$, and for all subdomains $j$ and states
\[
\theta_{ij,s} \mid \mu_i, \nu_i
\sim
\mathrm{Beta}(\mu_i \nu_i,\,(1-\mu_i)\nu_i).
\]
}
To estimate $\theta_{ij,s}$ while preserving the partial-pooling
structure (Figs.~\ref{fig_hierarchical structure},~\ref{fig_partial_pooling}),
all subdomain--state pairs within domain $D_i$ share common
domain-level hyperparameters $(\mu_i,\nu_i)$: $\mu_i \sim \mathrm{Beta}(a_i,b_i)$, $\nu_i \sim \mathrm{Gamma}(c_i,\mathrm{rate}=d_i)$, and
\[
\theta_{ij,s}\mid\mu_i,\nu_i
\sim
\mathrm{Beta}(\mu_i\nu_i,(1-\mu_i)\nu_i).
\]
This allows information to be pooled across both subdomains and interaction states while still allowing each $(j,s)$ pair to have its own reliability level.

\del{Here, $\mu_i \in (0,1)$ represents the domain-level mean
reliability, while $\nu_i > 0$ is a concentration parameter
controlling the dispersion of the subdomain--state reliabilities
around $\mu_i$. Large values of $\nu_i$ shrink $\theta_{ij,s}$
more tightly toward $\mu_i$, whereas small values allow
greater variability across subdomains and interaction states.
The role and impact of these hyperparameters have been
extensively discussed in~\cite{aghazadeh_hierarchical_2025}.}
Here, $\mu_i \in (0,1)$ denotes the domain-level mean reliability, while $\nu_i>0$ is a concentration parameter controlling the dispersion of the subdomain--state reliabilities around $\mu_i$. Larger values of $\nu_i$ induce stronger shrinkage of $\theta_{ij,s}$ toward $\mu_i$, whereas smaller values allow greater variability. The role of these hyperparameters is discussed in~\cite{aghazadeh_hierarchical_2025}.


The proposed model consists of three components: (i) a latent-state process describing the interaction condition, (ii) a state-dependent emission model linking the latent state to task outcomes, and (iii) a hierarchical prior structure sharing information across subdomains through partial pooling.

\paragraph{Generative mechanism within a session}
\del{Within a session in domain $D_i$, the generative 
process proceeds as:
\begin{enumerate}
    \item \textbf{Start of session}: The interaction starts in state $s$ 
    with probability $\delta_{i,s}$:
    \[
    Z_{i,1} \sim \delta_i, \quad \text{that is, } 
    p(Z_{i,1}=s) = \delta_{i,s}.
    \]
    \item \textbf{State evolution}: At each task, the interaction state 
    evolves according to the transition probabilities in $A_i$:
    \[
    Z_{i,t+1} \mid Z_{i,t} \sim A_i.
    \]
    \item \textbf{Task outcome}: Given the current state and the task's 
    subdomain, the observed outcome is drawn from a Bernoulli distribution 
    with success probability $\theta_{ij_t,s}$:
    \[
    Y_{i,t} \mid Z_{i,t},\, j_t \sim 
    \mathrm{Bernoulli}(\theta_{i j_t, Z_{i,t}}).
    \]
\end{enumerate}}
Within a session in domain $D_i$, the interaction starts in
a latent state drawn from the initial distribution $\delta_i$.
At each task, the latent state evolves according to the
transition matrix $A_i$. Conditional on the current latent
state and the subdomain label $j_t$, the task outcome
is generated from a Bernoulli distribution with success
probability $\theta_{i j_t, Z_{i,t}}$. This process repeats
throughout the session, inducing temporal dependence between
successive task outcomes through the evolving latent state.

\subsection{Modeling Assumptions and Practical Considerations}
\label{sec_hmm_remarks}

\begin{remark}[Role of task ordering and operational profile]
\label{remark_order_of_tasks}
\del{In sequential reliability assessment, the ordering of tasks may carry
important information about system performance over time. Although
tasks are drawn independently from the OP, which specifies only
their frequency of occurrence, successive task outcomes may be
statistically dependent because the interaction state evolves over
time due to accumulated context, errors, or workload conditions.
Models that assume independent task outcomes~\cite{aghazadeh_hierarchical_2025, luettgau_hibayes_2025}
treat observations as exchangeable and depend only on aggregated
success and failure counts, making them invariant to task ordering.
As a result, they cannot capture temporal clustering of failures
arising from evolving interaction conditions.}
In sequential reliability assessment, task ordering may carry important information about system performance over time. Although tasks are drawn independently from the OP, which specifies only their frequencies, successive outcomes may be statistically dependent because the interaction state evolves through accumulated context, errors, or workload conditions. Models that assume independent outcomes~\cite{aghazadeh_hierarchical_2025, luettgau_hibayes_2025} treat observations as exchangeable and depend only on aggregated success and failure counts, making them invariant to task ordering and unable to capture temporal clustering of failures.

\end{remark}

\begin{remark}[Selection of the Number of Latent States]
\label{remark_number_of_states}

\del{Determining the appropriate number of latent interaction states $K$ 
is a model selection problem. If $K$ is too small, the model may fail 
to capture meaningful sequential structure; if $K$ is too large, it may 
introduce unnecessary complexity and overfit the observed sequence. 
There are different approaches for selecting $K$, including marginal 
likelihood comparison (Bayes factors), cross-validation--based predictive 
assessment, etc.}

\del{We select $K$ using predictive model comparison based on the
\emph{out-of-sample sequential predictive log score}. For each candidate
value $K \in \{1,2,3,\ldots\}$, the HMM is fitted using the same prior
assumptions and evaluated on unseen future tasks. The predictive log
score measures how well the model predicts these outcomes: models that
assign higher probability to the observed future data obtain higher
scores. Because the log score is a strictly proper scoring rule
\cite{gneiting_strictly_2007}, maximizing this quantity favors models
with better predictive performance rather than unnecessary complexity.}

\del{We divide the observed task sequence $(j_{1:T}, y_{1:T})$ into two 
parts: a training segment $(j_{1:T_{\text{train}}}, 
y_{1:T_{\text{train}}})$ and a test segment 
$(j_{T_{\text{train}}+1:T}, y_{T_{\text{train}}+1:T})$ \del{(using 70\% of the tasks for training and the remaining 30\% for testing in the numerical examples in Sec.~\ref{sec_numerical_example})} (a 70/30 train--test split in Sec.~\ref{sec_numerical_example}). The split is performed chronologically, so that the model is trained on past observations and evaluated on future observations~\cite{gneiting_strictly_2007}, and the model is fitted using only the training data, yielding the posterior $p(\psi \mid y_{1:T_{\text{train}}}, j_{1:T_{\text{train}}})$.} 

Determining the number of latent states $K$ is a model
selection problem: if $K$ is too small, meaningful sequential structure
may be missed; if too large, it may 
introduce unnecessary complexity and overfit the observed sequence. Common approaches
include marginal likelihood comparison (Bayes factors) and
cross-validation--based predictive assessment.

We select $K$ using the \emph{out-of-sample sequential predictive log
score}. For each $K \in \{1,2, \ldots\}$, the HMM is fitted
under the same prior assumptions and evaluated on unseen future tasks.
The predictive log score measures how much probability the model assigns
to the observed future outcomes; because it is a strictly proper scoring
rule~\cite{gneiting_strictly_2007}, maximizing it favors predictive
accuracy over unnecessary complexity.

The observed sequence $(j_{1:T},y_{1:T})$ is split chronologically into
training and test segments,
$(j_{1:T_{\text{train}}},y_{1:T_{\text{train}}})$ and
$(j_{T_{\text{train}}+1:T},y_{T_{\text{train}}+1:T})$
(a 70/30 split in Sec.~\ref{sec_numerical_example}). The model is
fitted using only the training data, yielding the posterior
$p(\psi\mid y_{1:T_{\text{train}}},j_{1:T_{\text{train}}})$.
\del{The predictive distribution for the test sequence is then obtained by 
integrating over the posterior uncertainty:}
The predictive distribution for the test sequence is obtained by integrating over the posterior of the model parameters:
\begin{align}
& p(y_{T_{\text{train}}+1:T} \mid y_{1:T_{\text{train}}}, j_{1:T})
= \nonumber \\
&\int 
p(y_{T_{\text{train}}+1:T} \mid y_{1:T_{\text{train}}}, j_{1:T}, 
\psi)
\, 
p(\psi \mid y_{1:T_{\text{train}}}, j_{1:T_{\text{train}}})
\, d\psi \nonumber
\end{align}

where $\psi_i =
\{\theta_{ij,s},\, A_i,\, \mu_i,\, \nu_i\}$. The index $i$ emphasizes that these parameters are domain-specific. 
In practice, this integral is approximated using posterior samples 
$\{\psi_i^{(m)}\}_{m=1}^{N_{\text{samp}}}$, yielding
\[
\log \hat{p}
=
\log \left(
\frac{1}{N_{\text{samp}}}
\sum_{m=1}^{N_{\text{samp}}} 
p(y_{T_{\text{train}}+1:T}
\mid y_{1:T_{\text{train}}}, j_{1:T}, \psi_i^{(m)})
\right),
\]
where $N_{\text{samp}}$ denotes the number of posterior samples, 
$\psi_i^{(m)}$ is the $m$-th posterior draw of the parameter vector, 
and $T_{\text{train}}$ denotes the last time step included in the training data.

\del{If model comparison based on the predictive log score suggests $K=1$, 
this indicates that introducing multiple latent interaction states does 
not improve predictive performance. In that case, the sequential component 
of the HMM is unnecessary, and the model reduces to a hierarchical reliability 
model with independent task outcomes (e.g., \cite{aghazadeh_hierarchical_2025}). 
\del{Whether this occurs in practice depends on the domain and the structure of the observed sequence and is examined in Sec.~\ref{sec_numerical_example}.}}
If predictive model comparison selects $K=1$, introducing additional latent states does not improve predictive performance. In this case, the sequential component is unnecessary, and the model reduces to a hierarchical reliability model with independent task outcomes (e.g., \cite{aghazadeh_hierarchical_2025}).

\end{remark}

\begin{remark}[Interpreting the learned latent states]
\label{remark_interpreting_states}
To improve the interpretation of the learned latent states, we examine two outputs of the HMM.
First, the state-dependent success probabilities $\theta_{ij,s}$ indicate whether the inferred states correspond to different levels of performance. If the states are meaningful, they should be associated with different probabilities of success. 
Second, the posterior transition matrix describes how the latent state evolves across consecutive tasks. In particular, the self-transition probabilities indicate whether the model tends to remain in the same state over time. This helps determine whether the latent states represent persistent interaction regimes or merely short-lived statistical variations.
Together, these outputs answer two questions: (i) do the states differ in performance ($\theta_{ij,s}$), (ii) do they persist over time.  
\end{remark}

\begin{remark}[Time-Homogeneous State Transitions]
\label{remark_constant_matrix}

\del{In our model, the transition matrix $A_i$ is assumed to be
\emph{constant} over time. This means that the mechanism by
which the latent interaction state changes from one task to the
next does not depend on the task index $t$. If the system is
currently in a given state, the probabilities of remaining in
that state or transitioning to another state are the same
regardless of whether the interaction is at the beginning,
middle, or end of the session~\cite{cappe_inference_2005}.
The time-homogeneous assumption is adopted here for parsimony, as it substantially reduces the number of parameters to be estimated and is standard in HMM-based reliability modeling. This choice provides a tractable baseline for capturing sequential dependence while avoiding over-parameterization. 
A non-constant (time-varying)
transition matrix would be required if the interaction dynamics
themselves change during the session, for example due to fatigue
effects, gradual context drift, learning or adaptation over
time, or explicit dependence on the number of past failures. In
such cases, one would model $A_{i,t}$ instead of $A_i$, allowing
transition probabilities to vary with $t$.}

\del{In a fully Bayesian setting, if the transition matrix were 
allowed to vary over time (i.e., $A_{i,t}$ instead of $A_i$), 
one could place priors on each $A_{i,t}$ and update them 
sequentially as observations accumulate, allowing the 
transition mechanism itself to evolve during the session. 
This is not adopted here: $A_i$ is constant, and the Dirichlet prior on its rows (Remark~\ref{remark_delta_Ai}) reflects uncertainty in its fixed value, not temporal variation.}

In our model, the transition matrix $A_i$ is assumed to be
\emph{constant} over time, meaning that transition probabilities
do not depend on the task index $t$. Thus, if the system is in a
given state, the probabilities of remaining in that state or
transitioning to another are the same throughout the session
(beginning, middle, or end)~\cite{cappe_inference_2005}. This
time-homogeneous assumption is adopted for parsimony, reducing
the number of parameters to be estimated, and is standard in
HMM-based reliability modeling. It provides a tractable baseline
for capturing sequential dependence while avoiding
over-parameterization.

A time-varying transition matrix would be required if the
interaction dynamics themselves changed during a session, for
example due to fatigue, context drift, learning or adaptation,
or explicit dependence on the number of past failures. In that
case one would model $A_{i,t}$ instead of $A_i$, allowing
transition probabilities to vary with $t$.
In a fully Bayesian setting, priors could be placed on each
$A_{i,t}$ and updated sequentially as observations accumulate,
allowing the transition mechanism itself to evolve. This is not
adopted here: $A_i$ is fixed over time, and the Dirichlet prior
on its rows (Remark~\ref{remark_delta_Ai}) reflects uncertainty
about its value rather than temporal variation.

\end{remark}

\begin{remark}[Initial State Distribution and Transition Priors]
\label{remark_delta_Ai}
\del{The vector $\delta_i = (\delta_{i,1},\dots,\delta_{i,K})$ denotes the
initial distribution of the latent interaction state, where
$\delta_{i,s}=P(Z_{i,1}=s)$ represents the probability that the session
starts in state $s$. Since the state space contains $K$ possible states,
the distribution consists of $K$ probabilities that sum to one. 
For simplicity, the initial distribution is assumed uniform:
\[
\delta_{i,s}=\frac{1}{K}, \quad s=1,\dots,K,
\]
}

\del{Alternatively, the initial distribution $\delta_i$ may be treated as
unknown and estimated jointly with the transition matrix and emission
parameters within a Bayesian framework \cite{cappe_inference_2005}, which provides additional flexibility. However, as the sequence length increases, the influence of the initial
distribution diminishes due to the mixing behavior\footnote{This reflects the fact that, as the interaction progresses, the latent
state evolves according to the transition dynamics and gradually
becomes less dependent on its initial value.} of the underlying
Markov chain \cite{cappe_inference_2005}. In such cases, fixing
$\delta_i$ to a simple form (e.g., uniform) is often sufficient for
practical inference \cite{rabiner_tutorial_2002}.}

\del{To account for uncertainty in the fixed transition matrix $A_i$,
we place priors on its rows. Since each row is a probability
vector over $\{1,\dots,K\}$, we assign Dirichlet priors:
\[
(A_i)_{s,\cdot} \sim \mathrm{Dirichlet}(\lambda_{i,s,1},\dots,
\lambda_{i,s,K}), \quad s=1,\dots,K.
\]}
\del{Here $\lambda_{i,s,k} > 0$ encodes prior belief about transitions
from state $s$ to state $k$, with equal values yielding a
symmetric uninformative prior. The rows are assigned independent
priors, so the joint prior factorizes as
$p(A_i) = \prod_{s=1}^{K} p\bigl((A_i)_{s,\cdot}\bigr)$,
ensuring each row is updated separately in closed form via
Dirichlet--multinomial conjugacy.}
The vector $\delta_i=(\delta_{i,1},\dots,\delta_{i,K})$ denotes the
initial distribution of the latent interaction state, where
$\delta_{i,s}=P(Z_{i,1}=s)$ is the probability that the session starts
in state $s$. Since there are $K$ possible states, the probabilities
sum to one. For simplicity, we assume a uniform initial distribution: $\delta_{i,s}=\frac{1}{K}$,  $s=1,\dots,K.$

Alternatively, $\delta_i$ could be treated as unknown and estimated
jointly with the transition matrix and emission parameters within a
Bayesian framework~\cite{cappe_inference_2005}. However, as the
sequence length increases, the influence of $\delta_i$ diminishes due
to the mixing behavior\footnote{As the interaction progresses, the
latent state evolves according to the transition dynamics and becomes
less dependent on its initial value.} of the underlying Markov
chain~\cite{cappe_inference_2005}. Consequently, fixing $\delta_i$ to a
simple form (e.g., uniform) is often sufficient in
practice~\cite{rabiner_tutorial_2002}.

To account for uncertainty in the fixed transition matrix $A_i$, we
place Dirichlet priors on its rows:
\[
(A_i)_{s,\cdot}\sim
\mathrm{Dirichlet}(\lambda_{i,s,1},\dots,\lambda_{i,s,K}),
\quad s=1,\dots,K.
\]
Here $\lambda_{i,s,k}>0$ encodes prior belief about transitions from
state $s$ to state $k$, with equal values yielding a symmetric
uninformative prior. Assuming independent rows, $p(A_i)=\prod_{s=1}^{K}p\bigl((A_i)_{s,\cdot}\bigr)$
which allows each row to be updated separately via
Dirichlet--multinomial conjugacy.
\end{remark}

\begin{definition}[LLM reliability under sequential dependence]
\label{def_reliability_hmm}
Let $\mathcal{X}$ denote the input space of all possible tasks for a given
LLM, and let $\pi$ be the operational profile (OP), i.e., a
probability distribution over $\mathcal{X}$.
Consider a sequence of $n \geq 1$ future tasks
$\{x_\tau\}_{\tau=1}^{n}$ drawn independently according to $\pi$, and let
$I(x_\tau) \in \{0,1\}$ denote success (1) or failure (0).

Assume that task outcomes are sequentially dependent through a
latent state process $\{Z_\tau\}$ such that
\[
I(x_\tau) \mid (x_\tau, Z_\tau = s)
\sim \mathrm{Bernoulli}(\theta(x_\tau, s)),
\]
where $\theta(x_\tau, s)$ corresponds to $\theta_{ij,s}$ when task
$x_\tau$ belongs to subdomain $S_{ij}$ and the interaction is in
state $s$.
The latent state process $\{Z_\tau\}$ evolves as a
time-homogeneous, finite-state Markov chain with state space
$\{1,\ldots,K\}$, initial distribution, and transition kernel
specified in Sec.~\ref{sec_markov_model}; within the
hierarchical model, the domain-specific processes $Z_{i,t}$
defined there are realizations of this latent state process
restricted to domain $D_i$.
Assume furthermore that the task draws $\{x_\tau\}_{\tau=1}^{n}$
are independent of the latent state process $\{Z_\tau\}$.

Then, the LLM reliability is defined as
\[
R(n,\pi)
=
p\!\left(
\bigcap_{\tau=1}^{n}\{I(x_\tau)=1\}
\right),
\]
where the probability is taken jointly over the future task sequence
$\{x_\tau\}_{\tau=1}^{n}$ and the latent state process $\{Z_\tau\}$.

\end{definition}

\begin{remark}[Dependence on model specification and parameters]
\label{remark_parameter_dependence}
\del{The quantity $R(n,\pi)$ also depends on the model specification and
parameters governing the latent state process and emission probabilities,
including the transition kernel and the success probabilities
$\{\theta_{ij,s}\}$. In the present formulation, the initial
distribution is treated as fixed (Sec.~\ref{sec_markov_model}), rather than inferred. This dependence is
suppressed in the notation, since the unknown quantities are treated as
random variables under the Bayesian framework. Their posterior
uncertainty is propagated to the reliability quantities in
Theorem~\ref{thm_posterior_reliability}.}

For simplicity, the dependence of $R(n,\pi)$ on the latent-state model and its parameters ($\psi_i =
\{\theta_{ij,s},\, A_i,\, \mu_i,\, \nu_i\}$) is suppressed in the notation. Under the Bayesian framework, these quantities are treated as random variables and their uncertainty is propagated to reliability quantities in Theorem~\ref{thm_posterior_reliability}.
\end{remark}

\robab{I merged two Remarks 7 and 8 in to Remark 7.}
\del{\begin{remark}[OP weights]
    The aggregation of reliability across subdomains and domains is governed by OP weights derived from the distribution $\pi$. Within domain $D_i$, let $\Omega_{ij} \geq 0$ denote the weight of subdomain $S_{ij}$ (the proportion of tasks drawn from subdomain $S_{ij}$ among all tasks in $D_i$) with $\sum_{j=1}^{n_i} \Omega_{ij} = 1$. At the domain level, let $W_i \geq 0$ denote the proportion of all tasks attributed to domain $D_i$, with $\sum_{i=1}^{m} W_i = 1$.
\end{remark}}

\del{\begin{remark}[Connecting $\pi$ to operational weights]
The OP $\pi$ is a distribution over the input space $\mathcal{X}$.
Within the hierarchical model, it induces a two-level mixture
structure: tasks are assigned to domains with probabilities
$\{W_i\}$ and, conditional on domain $D_i$, to subdomains with
probabilities $\{\Omega_{ij}\}$.
\\
These weights connect the abstract OP $\pi$ to the quantities
used in Theorem~\ref{thm_posterior_reliability}. In particular,
they determine the state-conditional average success probability
$\bar{\theta}_{i,s} = \sum_{j=1}^{n_i} \Omega_{ij}\theta_{ij,s}$,
which enters the domain-level reliability recursion.
\end{remark}
}

\begin{remark}[Operational-profile weights]
The OP $\pi$ is a distribution over the input space $\mathcal{X}$.
Within the hierarchical model, it induces a two-level mixture
structure: tasks are assigned to domains with probabilities
$\{W_i\}$ and, conditional on domain $D_i$, to subdomains with
probabilities $\{\Omega_{ij}\}$.

Here, $W_i\ge0$ denotes the proportion of all tasks attributed
to domain $D_i$, with $\sum_{i=1}^{m}W_i=1$, while
$\Omega_{ij}\ge0$ denotes the proportion of tasks in $D_i$
belonging to subdomain $S_{ij}$, with
$\sum_{j=1}^{n_i}\Omega_{ij}=1$.
These weights determine the state-conditional average success
probability $\bar{\theta}_{i,s}
=
\sum_{j=1}^{n_i}\Omega_{ij}\theta_{ij,s}$, 
which enters the domain-level reliability recursion in
Theorem~\ref{thm_posterior_reliability}.
\end{remark}

\del{Definition~\ref{def_reliability_hmm} provides the conceptual
definition of LLM reliability under sequential dependence.
To instantiate this definition within the proposed hierarchical
HMM, each task $x \in \mathcal{X}$ is represented by a
subdomain label $j$ and a binary outcome $Y \in \{0,1\}$,
while dependence across tasks is captured through the latent
state process. Because the corresponding reliability quantities
depend on unknown model parameters ($\psi$), they must be inferred
from data. Accordingly,
Theorem~\ref{thm_postrior_for_HMM} derives the posterior
distribution of the model parameters, and
Theorem~\ref{thm_posterior_reliability} derives the posterior
distributions of the resulting reliability measures,
instantiating Definition~\ref{def_reliability_hmm} within each
fixed domain under the hierarchical representation induced by
the operational weights $\Omega_{ij}$. At the overall LLM
level, the quantity $R_L(n;\psi)$ should be interpreted as an
operational-profile-weighted aggregation of domain-level
reliabilities, rather than as the full joint reliability over
arbitrary cross-domain task sequences; see
Remark~\ref{remark_interpret_LLM_reliability} for a precise
statement of this scope limitation.}

Definition~\ref{def_reliability_hmm} provides the conceptual
definition of LLM reliability under sequential dependence. In the
proposed hierarchical HMM, each task $x\in\mathcal{X}$ is represented
by a subdomain label $j$ and binary outcome $Y\in\{0,1\}$, while
dependence across tasks is captured by the latent state process.
Because the reliability quantities depend on unknown parameters $\psi_i =
\{\theta_{ij,s},\, A_i,\, \mu_i,\, \nu_i\}$, they must be inferred from data. Accordingly,
Theorem~\ref{thm_postrior_for_HMM} derives the posterior
distribution of $\psi$, and
Theorem~\ref{thm_posterior_reliability} derives the posterior
distributions of the resulting reliability measures, thereby
instantiating Definition~\ref{def_reliability_hmm} within each
domain through the operational weights $\Omega_{ij}$. At the LLM
level, $R_L(n, \pi;\psi)$ is an OP-weighted aggregation
of domain-level reliabilities rather than the full joint reliability
over arbitrary cross-domain task sequences; see
Remark~\ref{remark_interpret_LLM_reliability}.

\begin{theorem}[Posterior distribution for the hierarchical HMM parameters]
\label{thm_postrior_for_HMM}

Fix a domain $D_i$ and consider the hierarchical Hidden Markov
Model defined in Sec.~\ref{sec_markov_model} and Fig.~\ref{fig_hierarchical structure}, with parameter vector $\psi_i =
\{\theta_{ij,s},\, A_i,\, \mu_i,\, \nu_i\}$
where $(\mu_i,\nu_i)$ are the domain-level hyperparameters
governing the hierarchical prior on $\theta_{ij,s}$.

Given the observed sequence
$\{(j_t, Y_{i,t})\}_{t=1}^T$, the posterior distribution of
$\psi_i$ satisfies
\begin{align}
& p(\psi_i \mid Y_{i,1:T}, j_{1:T})
\propto \nonumber \\
& \hspace{1cm} \mathcal{L}_i(\theta_i,A_i)\,
p(\theta_i \mid \mu_i,\nu_i)\,
p(\mu_i,\nu_i)\,
p(A_i),
\label{eq_hmm_posterior}
\end{align}
where the marginal likelihood is
\begin{align}
&\mathcal{L}_i(\theta_i,A_i)
= \nonumber \\
& \sum_{z_{1:T}\in\mathcal{Z}^T}
p(Z_{i,1:T}=z_{1:T},Y_{i,1:T}=y_{1:T}
\mid j_{1:T},\theta_i,A_i,\delta_i).
\label{eq_hmm_marginal}
\end{align}
The marginal likelihood has no closed form due to the latent state sequence, but is finite.

\end{theorem}

\begin{theorem}[Posterior distributions of reliability]
\label{thm_posterior_reliability}
 
Fix a domain $D_i$ and consider the hierarchical HMM defined in
Theorem~\ref{thm_postrior_for_HMM}, with parameter vector
$\psi_i := \{\theta_{ij,s},\, A_i,\, \mu_i,\, \nu_i\}$
and posterior distribution $p(\psi_i \mid Y_{i,1:T}, j_{1:T})$.
Under this posterior, $\psi_i$ is treated as a random variable.
Consequently, the reliability quantities defined below, being
functions of $\psi_i$, are also random variables with induced
posterior distributions.
 
Let $\rho_i(s';\psi_i)$ denote the one-step-ahead predictive
distribution of the latent state associated with the next task.
For a continuing interaction session, 
\begin{align}
 & \rho_i(s';\psi_i)
  = \nonumber \\
  & \sum_{s=1}^{K}
  p\!\left(Z_{i,T}=s \;\middle|\;
    Y_{i,1:T},\, j_{1:T},\, \psi_i\right)
  (A_i)_{s,s'},
  \quad s' \in \{1,\ldots,K\} \nonumber
\end{align}
and for a new session $\rho_i(s';\psi_i)=\delta_{i,s'}$.
 
\medskip
\noindent
\textbf{Subdomain-level reliability 
}
For each subdomain $S_{ij}$, define the state-conditional
$n$ future task reliability by
\[
  R_{ij,s}(1;\psi_i) = \theta_{ij,s},
\]
\del{and, for $n \ge 2$,}
\[
  R_{ij,s}(n;\psi_i)
  =
  \theta_{ij,s}
  \sum_{s'=1}^{K} (A_i)_{s,s'}\, R_{ij,s'}(n-1;\psi_i) \qquad n \ge 2
\]
Here $\theta_{ij,s}$ is the probability that a single task from
subdomain $S_{ij}$ succeeds when the interaction is in latent
state $s$, and $R_{ij,s}(n;\psi_i)$ is the probability of
$n$ consecutive successes conditional on starting in state $s$
and on the event that all $n$ future tasks belong to subdomain
$S_{ij}$.
The corresponding \emph{conditional} subdomain reliability is
\[
  R_{ij}(n;\psi_i)
  =
  \sum_{s=1}^{K} \rho_i(s;\psi_i)\, R_{ij,s}(n;\psi_i).
\]

\medskip
\noindent
\textbf{Domain-level reliability under mixed-subdomain sequences.}
We now define the domain-level reliability measure
within domain $D_i$ for all $n \ge 1$.
We assume that the subdomain label $j_t$ is drawn independently
from the operational weights $\Omega_{i\cdot}$ at each step,
independently of the current latent state $Z_{i,t}$. This
assumes that task arrival is governed by the operational
profile, while the latent state governs performance dynamics.
Define the \emph{state-conditional average success probability} $ \bar{\theta}_{i,s}
  =
  \sum_{j=1}^{n_i} \Omega_{ij}\, \theta_{ij,s},$
which is the expected probability of success for a task drawn
from domain $D_i$ according to the OP,
conditional on the latent state being $s$. This is a
deterministic function of the existing model parameters and
operational weights.
 
Define the state-conditional $n$-task reliability under
mixed-subdomain sequences by
\[
  \bar{R}_{i,s}(1;\psi_i) = \bar{\theta}_{i,s},
\]
\[
 \bar{R}_{i,s}(n;\psi_i)
  =
  \bar{\theta}_{i,s}
  \sum_{s'=1}^{K} (A_i)_{s,s'}\, \bar{R}_{i,s'}(n-1;\psi_i) \;\; n \ge 2.
\]

Thus, conditional on starting in state $s$, the probability of
$n$ consecutive successes under mixed-subdomain usage equals the
expected probability of success on the current task,
$\bar{\theta}_{i,s}$, multiplied by the probability of $n-1$
further consecutive successes after the state transition. 
The domain reliability is: 
\[
  R_i(n;\psi_i)
  =
  \sum_{s=1}^{K} \rho_i(s;\psi_i)\, \bar{R}_{i,s}(n;\psi_i).
\]
\del{For $n=1$ this reduces to}
\[
  R_i(1;\psi_i)
  =
  \sum_{j=1}^{n_i} \Omega_{ij}
  \sum_{s=1}^{K} \rho_i(s;\psi_i)\, \theta_{ij,s}, \qquad n=1
\]
which coincides with the single-task success probability under
Definition~\ref{def_reliability_hmm}. For $n > 1$,
consistency with Definition~\ref{def_reliability_hmm} within
domain $D_i$ holds under the stated independence of $j_t$
from $Z_{i,t}$, as verified in the proof
(Appendix~\ref{sec_proof_theorem}).
 
\medskip
\noindent
\textbf{LLM-level reliability.}
The overall LLM reliability is
\[
  R_L(n;\psi)
  =
  \sum_{i=1}^{m} W_i\, R_i(n;\psi_i),
  \qquad
  \psi = \{\psi_i\}_{i=1}^{m},
\]
where $W_i \ge 0$ are domain-level operational weights with
$\sum_{i=1}^{m} W_i = 1$. This is a usage-weighted aggregation
of domain-level mixed-subdomain reliabilities: within each
domain, tasks are drawn from the OP
$\Omega_{i\cdot}$ at each step, and the results are averaged
across domains using $W_i$. It does not represent the joint success probability for sequences with arbitrary domain switching at each step. 
 
\medskip
\noindent
Since all reliability quantities are measurable functions of
$\psi_i$, their posterior distributions are obtained by
propagating the parameter posterior
$p(\psi_i \mid Y_{i,1:T}, j_{1:T})$ through these functionals.
Consequently, $R_{ij}(n;\psi_i)$, $R_i(n;\psi_i)$, and
$R_L(n;\psi)$ admit posterior distributions induced by the
parameter posterior in Theorem~\ref{thm_postrior_for_HMM}.
In particular, for any $x \in [0,1]$,
\begin{align}
  F_{R_{ij}(n)}(x)
  &=
  p \bigl(R_{ij}(n;\psi_i) \le x
    \mid Y_{i,1:T},\, j_{1:T}\bigr), \nonumber \\
  F_{R_i(n)}(x)
  &=
  p \bigl(R_i(n;\psi_i) \le x
    \mid Y_{i,1:T},\, j_{1:T}\bigr), \nonumber\\
  F_{R_L(n)}(x)
  &=
  p \bigl(R_L(n;\psi) \le x
    \mid \{Y_{i,1:T},\, j_{1:T}\}_{i=1}^{m}\bigr) \nonumber 
\end{align}
These posteriors lack closed forms and are approximated via
Monte Carlo sampling.
 
\end{theorem}
 
\begin{remark}[Interpretation of domain- and LLM-level reliability]
\label{remark_interpret_LLM_reliability}

\del{The domain-level reliability $R_i(n;\psi_i)$ represents the probability
that $n$ consecutive tasks drawn from domain $D_i$ according to the
OP all succeed, where at each step the subdomain label
is drawn independently from $\Omega_{i\cdot}$ and dependence across
tasks arises solely through the latent state process. Within domain
$D_i$, this construction instantiates
Definition~\ref{def_reliability_hmm} when the OP is
restricted to subdomains of $D_i$.}

\del{In contrast, the LLM-level quantity $R_L(n;\psi)$ is not, in the present
formulation, the full joint reliability $R(n,\pi)$ from
Definition~\ref{def_reliability_hmm}$.$ Rather, it is an
OP-weighted aggregation of the domain-level
reliabilities $\{R_i(n;\psi_i)\}_{i=1}^m$. It therefore does not model
task sequences in which the domain itself may switch across time steps
according to the global OP. Extending to full cross-domain mixing requires an additional stochastic model and is left for future work.} 

\del{The independence assumption $j_t \perp Z_{i,t}$ is the hierarchical
counterpart of the assumption in
Definition~\ref{def_reliability_hmm} that future task draws are
independent of the latent state process. 
The OP governs which
tasks arrive (an external property of the deployment environment), while
the latent state captures how well the model handles them (an internal
performance dynamic). We adopt it as a tractable baseline,
treating the OP as external to model behavior
(Sec.~\ref{sec_discussion_conclusion}).}

\del{The subdomain-level quantity $R_{ij}(n;\psi_i)$ is retained as a
conditional diagnostic (within subdomain): it gives the probability of
$n$ consecutive successes when all tasks come from subdomain $S_{ij}$,
and is useful for identifying which task types are most sensitive to
interaction state or most responsible for degraded sustained
performance. The parameters $\theta_{ij,s}$ are still estimated
individually for every (subdomain $j$, state $s$) pair; no information
is lost at the estimation stage. The averaging into
$\bar{\theta}_{i,s}$ occurs only when computing the deployment-relevant
$n$-step reliability, and the within-subdomain recursion
$R_{ij,s}(n;\psi_i)$ remains available for targeted subdomain analysis.}

The domain-level reliability $R_i(n;\psi_i)$ is the probability that
$n$ consecutive tasks drawn from domain $D_i$ according to the OP all
succeed, where subdomain labels are sampled independently from
$\Omega_{i\cdot}$ and dependence arises solely through the latent state
process. Within domain $D_i$, this corresponds to
Definition~\ref{def_reliability_hmm} restricted to the subdomains of
$D_i$.

By contrast, the LLM-level quantity $R_L(n;\psi)$ is not the full joint
reliability $R(n,\pi)$ of Definition~\ref{def_reliability_hmm}, but an
OP-weighted aggregation of the domain-level reliabilities
$\{R_i(n;\psi_i)\}_{i=1}^m$. It therefore does not model sequences in
which domains themselves switch over time according to the global OP.
Extending the framework to full cross-domain mixing requires an
additional stochastic model and is left for future work.
\end{remark}

\begin{remark}[Separation of task arrivals and performance dynamics]
The assumption $j_t \perp Z_{i,t}$ is the hierarchical analogue of the
independence assumption in Definition~\ref{def_reliability_hmm}. The OP
governs task arrivals, while the latent state governs performance
dynamics. We adopt this separation as a tractable baseline and treat
the OP as external to model behavior
(Sec.~\ref{sec_discussion_conclusion}).
\end{remark}

\begin{remark}[Role of subdomain-level reliability]
The subdomain-level quantity $R_{ij}(n;\psi_i)$ is retained as a
conditional diagnostic: it gives the probability of $n$ consecutive
successes when all tasks belong to subdomain $S_{ij}$ and helps identify
task types most sensitive to interaction state or degraded sustained
performance. Parameters $\theta_{ij,s}$ are still estimated separately
for each (subdomain, state) pair; averaging into
$\bar{\theta}_{i,s}$ occurs only when computing deployment-relevant
$n$-step reliability, while the within-subdomain recursion
$R_{ij,s}(n;\psi_i)$ remains available for targeted analysis.
\end{remark}

\section{Evaluation}
\label{sec_numerical_example}

\subsection{Problem Setting and Data Construction}
\label{sec_problem_setting}

\paragraph{Domains and subdomains of the model}
We consider two domains—\emph{Reasoning} and \emph{Coding}—each comprising two benchmark-based subdomains. The Reasoning domain includes \textbf{BoolQ} (binary yes/no question answering) and \textbf{RACE-H} (multiple-choice reading comprehension), while the Coding domain includes \textbf{DS1000} (data science programming tasks) and \textbf{MBPP} (Python programming problems).

\paragraph{Task outcomes and implementation details}
Tasks from all datasets are evaluated using  Anthropic Claude (\texttt{claude-sonnet-4.6}) and OpenAI (\texttt{gpt-5.2}).
For BoolQ and RACE-H, correctness is determined by matching predicted answers with ground truth. For DS1000 and MBPP, correctness is determined by executing the generated code and verifying test cases. Unless explicitly stated otherwise, we apply the following ordering scheme: Reasoning domain-BoolQ followed by RACE-H; Coding domain-DS1000 followed by MBPP. We refer to this as the original scheme. This ordering is benchmark-derived and is used to construct controlled sequential sessions for evaluating the proposed methodology; it is not intended to represent naturally occurring user interaction sequences in deployment. Table~\ref{tab_ordering_sensitivity} reports results under alternative task orderings.

\paragraph{Sequential-session protocol}
For each domain, benchmark tasks were executed within a single continuous conversation. Each task was submitted as a new user message while preserving the complete conversation history, so task $t$ was evaluated using the accumulated interaction history from tasks $1,\ldots,t-1$. The system prompt remained fixed throughout the session, with no intermediate resets within a domain. Context was reset only when switching between domains, so the Reasoning and Coding domains were evaluated in separate conversations. No context-window truncation occurred because the accumulated context remained below the model's maximum context length.

For HMM inference, each domain session is represented as an observed sequence
$\{(j_t,y_t)\}_{t=1}^{T}$, where $j_t$ denotes the subdomain
(BoolQ/RACE-H for the Reasoning and DS-1000/MBPP for the Coding),
and $y_t\in\{0,1\}$ is the binary task outcome obtained by automated evaluation.
Table~\ref{tab_session_example} illustrates this protocol using the first three
tasks from a Coding-domain (DS-1000) session. The ``Tokens'' column reports the
approximate size of the accumulated conversation history before each task. Its
increasing values show that previous prompts and model
responses remain available to the LLM throughout the session, demonstrating
that context is retained.

\begin{table}[t]
\centering
\caption{Example excerpt from a continuous OpenAI Coding domain (DS-1000) session.}
\label{tab_session_example}
\scriptsize
\begin{tabular}[!htbp]{p{0.13 cm} p{0.13cm} p{0.5cm} p{2.4cm} p{2.9cm} p{0.13cm}}
\hline
$t$ &
$j_t$ &
Tokens &
{Prompt} &
LLM response&
$y_t$\\
\hline
1 & DS & 0 &
\makecell[l]{Draw a full line\\ from $(0,0)$ to $(1,2)$.} &
\makecell[l]{\texttt{plt.plot}\\\texttt{([0,1],[0,2])...}} &
0\\

2 & DS & 103 &
\makecell[l]{Count connected \\ regions above 0.75.} &
\makecell[l]{\texttt{mask = img >}\\ \texttt{threshold; ...}} &
1\\

3 & DS & 442 &
\makecell[l]{Find centres of mass\\ of connected regions.} &
\makecell[l]{\texttt{centers = ndimage.}\\ \texttt{center\_of\_mass(...)}} &
0\\
\hline
\end{tabular}
\end{table}

\paragraph{Operational profile} 
\del{Within each
domain $D_i$, the weights $\Omega_{ij}$ are used to form the
state-conditional average success probability
$\bar{\theta}_{i,s} = \sum_{j=1}^{n_i} \Omega_{ij}\theta_{ij,s}$, entering the domain-level reliability recursion directly.
At the LLM level, domain-level reliabilities are aggregated
using the domain weights $W_i$.}

The baseline experiments use OP weights proportional to benchmark dataset sizes. Since these weights are determined by benchmark composition rather than real-world usage, a sensitivity analysis to alternative OP assumptions is also performed (Sec.~\ref{sec_evaluation_reliability}).

\paragraph{Model comparison setup}
\del{We compare reliability under the \textbf{HMM} with the \textbf{HIP} model~\cite{aghazadeh_hierarchical_2025}, which assumes independent outcomes; both share the same hierarchical structure (Fig.~\ref{fig_hierarchical structure}). For the HMM, the number of latent states $K$ is selected per domain using the predictive log score (Sec.~\ref{sec_evaluation_reliability}).}
We compare the HMM with the HIP model~\cite{aghazadeh_hierarchical_2025}, which assumes independent outcomes; both share the same hierarchical structure (Fig.~\ref{fig_hierarchical structure}). For the HMM, the number of latent states $K$ is selected separately for each domain using the predictive log score (Sec.~\ref{sec_evaluation_reliability}).

\paragraph{Reliability estimation}
\del{Given the observed sequences $\{(j_t, y_t)\}_{t=1}^T$, posterior
distributions of reliability are obtained via MCMC sampling.
For each posterior sample, the state-conditional average success
probability $\bar{\theta}_{i,s}^{(m)} = \sum_j \Omega_{ij}
\theta_{ij,s}^{(m)}$ is computed and fed into the $n$-step
recursion, weighted by the predictive state distribution to
obtain domain-level reliability. Subdomain-level quantities
$R_{ij}(1;\psi_i)$ are computed separately for Fig.~\ref{fig_subdomain_claude_gpt}.
In the HIP setting, subdomain reliability is obtained directly
and aggregated using operational-profile weights. Reliability
is evaluated at $n=1$ for Figs.~\ref{fig_pdf_claude_gpt} and~\ref{fig_PDF_HIP_HMM}, Table.~\ref{tab_ordering_sensitivity}, and at
$n=1,\dots,20$ for Fig.~\ref{fig_Reliability_n_future_task}.}
The posterior distribution in Theorem~\ref{thm_postrior_for_HMM}
has no closed form and is approximated via MCMC. We use a Gibbs
sampler with conjugate updates for $\theta_{ij,s}$ and the rows of $A_i$, a Metropolis--Hastings step for the non-conjugate
hyperparameters $(\mu_i,\nu_i)$, and forward--backward sampling
for the latent states. The sampler is run for $12{,}000$ iterations,
with the first $2{,}000$ discarded as burn-in and every second sample
retained, yielding $5{,}000$ posterior samples. Random seeds are
fixed, and convergence is assessed through trace plots and stability
of posterior summaries.

For each posterior sample, the state-conditional average success
probability
$\bar{\theta}_{i,s}^{(m)}=\sum_j\Omega_{ij}\theta_{ij,s}^{(m)}$
is computed and used in the $n$-step recursion, weighted by the
predictive state distribution, to obtain domain-level reliability.
Subdomain-level quantities $R_{ij}(1;\psi_i)$ are computed separately (Fig.~\ref{fig_subdomain_claude_gpt}). Under HIP, subdomain
reliabilities are obtained directly and aggregated using
OP weights. Except for Fig.~\ref{fig_Reliability_n_future_task}, reliability is evaluated at n=1 in all tables and figures, where applicable.

\del{The posterior distribution in Theorem~\ref{thm_postrior_for_HMM}
does not admit a closed-form expression and is approximated via
Markov chain Monte Carlo (MCMC). The sampler follows a Gibbs
scheme with conjugate updates for $\theta_{ij,s}$ and the rows
of $A_i$, and a Metropolis--Hastings step for $(\mu_i,\nu_i)$,
whose full conditional is non-conjugate. Latent state sequences
are sampled via forward--backward procedures within this scheme.
We run $12{,}000$ iterations, discard the first $2{,}000$ as burn-in,
and retain every second sample, yielding $5{,}000$ posterior samples.
Random seeds are fixed for reproducibility. Convergence of the Gibbs sampler
was monitored informally through trace plot inspection and stability of
posterior summaries.}

\subsection{Numerical Results}
\label{sec_evaluation_reliability}

\paragraph{Determining the number of latent interaction states $K$}
\label{sec_evaluation_state_number}

\del{$K$ is selected using out-of-sample sequential predictive log
scores (Remark~\ref{remark_number_of_states}), choosing the
smallest $K$ beyond which improvements are negligible
(Table.~\ref{tab_log_score}). For Claude, $K=3$ is selected for
Reasoning and $K=1$ for Coding. For OpenAI, $K=1$ is selected
for Reasoning and $K=3$ for Coding. The gain from $K=3$ to
$K=4$ is negligible in all nontrivial domains, and the
increased computational cost ($\mathcal{O}(TK^2)$,
i.e., $1.78\times$ from $K=3$ to $K=4$) further supports
$K=3$ as the parsimonious choice.}

The number of latent states is selected using chronological out-of-sample predictive log scores (Remark~\ref{remark_number_of_states}). Table~\ref{tab_log_score} shows that multiple latent states improve predictive performance for Claude Reasoning and OpenAI Coding. However, the gains diminish as $K$ increases: the improvement from $K=3$ to $K=4$ (Claude) is small relative to the increase in model complexity and computational cost $(\mathcal{O}(TK^2))$, corresponding to a 1.78$\times$ increase in computation. We therefore adopt ($K=3$) as a parsimonious choice for domains exhibiting temporal dependence. This choice balances predictive performance and complexity and should not be interpreted as identifying the true number of latent interaction states.

The predictive log-score analysis motivates the choice of $K=3$ for the Claude Reasoning domain. We therefore examine the resulting fitted HMM by analyzing its learned transition matrix and state-dependent success probabilities to assess whether the inferred latent states correspond to persistent interaction regimes, as discussed in Remark~\ref{remark_interpreting_states}. This analysis is presented for the Claude Reasoning domain at the end of Sec.~\ref{sec_evaluation_reliability}.

\begin{table}[!htbp]
\centering
\caption{Predictive log scores for different values of $K$.}
\label{tab_logscores}
\begin{tabular}{ccccc}
\toprule
& \multicolumn{2}{c}{\textbf{Claude}} 
& \multicolumn{2}{c}{\textbf{OpenAI}} \\
\cmidrule(lr){2-3} \cmidrule(lr){4-5}
$K$ & Reasoning & Coding & Reasoning & Coding \\
\midrule
1 & -72.427 & \textbf{-21.539} & \textbf{-70.203} & -25.633 \\
2 & -72.228 & -23.120          & -70.374          & -25.404 \\
3 & \textbf{-72.215} & -23.047 & -70.379          & \textbf{-25.363} \\
4 & -72.212 & -22.950          & -70.369          & -25.341 \\
\bottomrule
\end{tabular}
\label{tab_log_score}
\end{table}

\paragraph{\del{HMM-based reliability comparison between Claude and OpenAI} HMM-based reliability: Claude vs OpenAI}
\del{Figure~\ref{fig_pdf_claude_gpt} compares PDF for the next future task across different levels for Claude and OpenAI.
At the subdomain level, reasoning subdomains show high reliability for both models. OpenAI's reasoning estimates are tightly concentrated because 
$K=1$ was selected, meaning a single latent state governs all 
outcomes and no state-switching uncertainty inflates the posterior 
variance. Claude's reasoning, with $K=3$, produces a wider 
posterior, as the reliability estimate integrates over three 
latent states with potentially different success probabilities. In coding, Claude has a single
state with moderate reliability, while OpenAI shows three
states with lower and more uncertain estimates
(stronger sequential structure). At the domain and LLM levels, reasoning dominates due to higher weight, with Claude showing slightly greater uncertainty.}
Figure~\ref{fig_pdf_claude_gpt} compares the PDFs of next-task reliability across hierarchical levels for Claude and OpenAI. At the subdomain level, both models exhibit high reliability in reasoning. OpenAI's reasoning estimates are tightly concentrated because $K=1$, whereas Claude's wider posterior reflects uncertainty across three latent states ($K=3$). In coding, Claude has a single state with moderate reliability, while OpenAI exhibits three states with lower and more uncertain estimates, indicating stronger sequential structure. At the domain and LLM levels, reasoning dominates because of its higher weight, with Claude showing slightly greater uncertainty.
\begin{figure}[!htbp]
    \centering
\begin{subfigure}[t]{0.45
    \textwidth}
	\centering	\includegraphics[width=0.7 \linewidth]{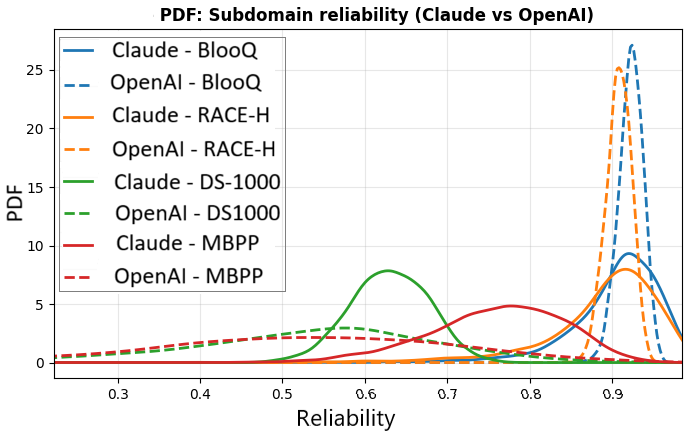}
	\caption{Posterior PDFs of reliability in subdomains under HMM.}	\label{fig_subdomain_claude_gpt}
\end{subfigure}
	\hspace{0.3 cm}
\begin{subfigure}[t]{0.45\textwidth}
	\centering
\includegraphics[width=0.75\linewidth]{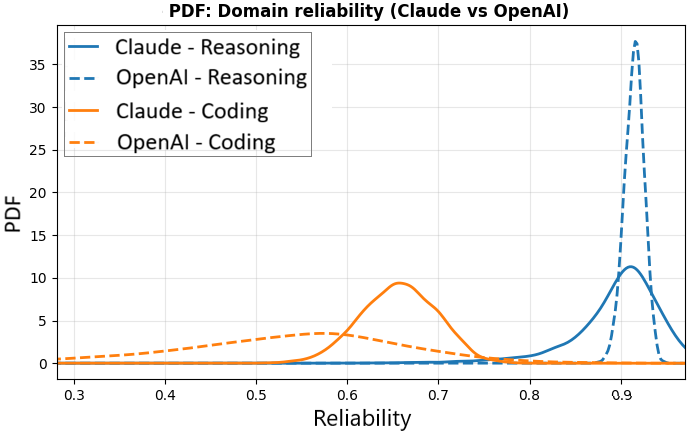}
	\caption{Posterior PDFs of reliability in domain level under HMM.}	
    \label{fig_domain_claude_gpt}
\end{subfigure}
\hspace{0.3 cm}
\begin{subfigure}[t]{0.45\textwidth}
	\centering
\includegraphics[width=0.75\linewidth]{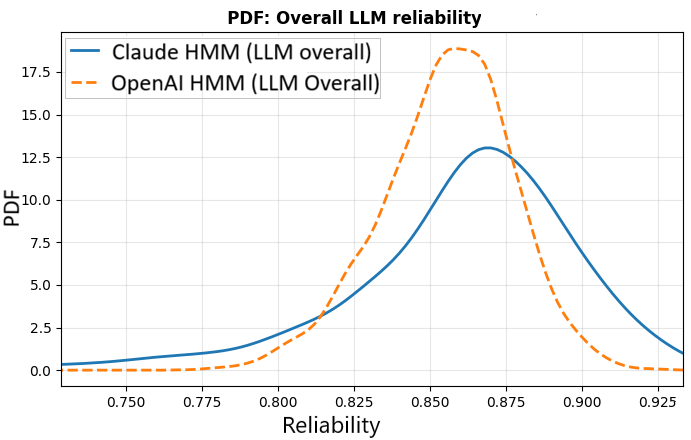}
	\caption{Posterior PDF of reliability in LLM level under HMM.}
    \label{fig_LMM_claude_gpt}
\end{subfigure}
\caption{Comparison of posterior reliability distributions at different levels under the HMM for Claude and OpenAI.}
\label{fig_pdf_claude_gpt}
\end{figure}

\paragraph{Effect of the task horizon on reliability under HMM}
\del{Figure~\ref{fig_Reliability_n_future_task} shows posterior mean LLM reliability with $95\%$ credible intervals for $n=1,\ldots,20$. Reliability decreases monotonically for both models—from similar values at $n=1$ to $0.130$ and $0.146$ at $n=20$—with Claude exhibiting wider intervals, indicating greater uncertainty.}
Figure~\ref{fig_Reliability_n_future_task} shows posterior mean LLM reliability and 95\% credible intervals for $n=1,\ldots,20$. Reliability decreases monotonically for both models, reaching $0.130$ and $0.146$ at $n=20$, with Claude exhibiting wider intervals and therefore greater uncertainty.

\begin{figure}[!h]
    \centering
    \includegraphics[width=0.65\linewidth]{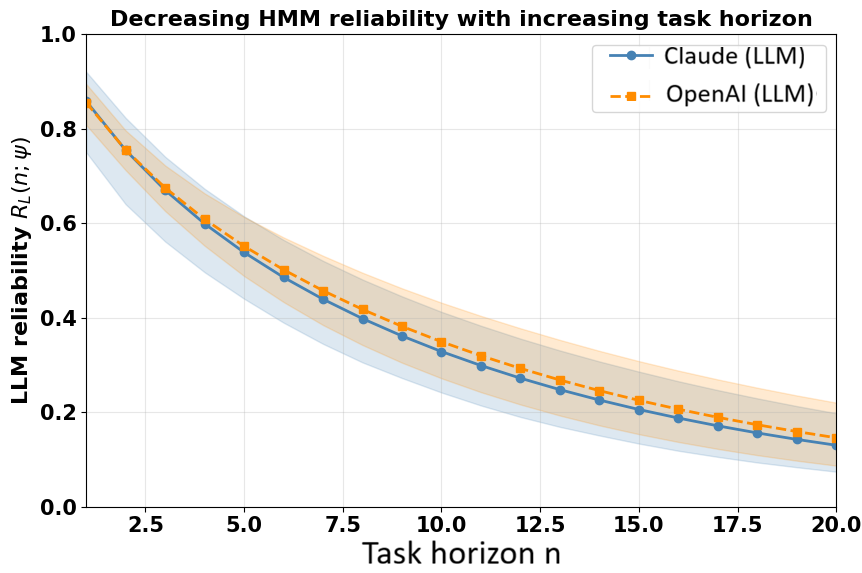}
    \caption{HMM reliability decreases with $n$ in both models.}
    \label{fig_Reliability_n_future_task}
\end{figure}

\paragraph{Comparison between HIP (independent outcomes) and HMM reliability distributions}

\del{Figure~\ref{fig_PDF_HIP_HMM} compares posterior reliabilities under the HIP model~\cite{aghazadeh_hierarchical_2025} and HMM.}

\del{In the coding domain (Figure~\ref{fig_PDF_coding_HIP_HMM}), the distributions nearly coincide, indicating that latent states do not materially affect reliability estimates. This aligns with $K=1$, suggesting approximately independent outcomes.}

\del{In the reasoning domain (Fig.~\ref{fig_PDF_reasoning_HIP_HMM}), the HIP yields a peaked distribution, indicating high confidence in a narrow range, while the HMM produces a wider one. This reflects sequential dependence: success probabilities vary with the latent state, and ignoring this leads HIP to treat correlated observations as independent, understating uncertainty. This effect propagates to the LLM level (Figure~\ref{fig_PDF_LLM_HIP_HMM}), where HIP yields a narrow distribution while HMM produces a broader, heavy-tailed one.}
Figure~\ref{fig_pdf_hmm_hip} compares posterior reliabilities under the HIP (i.i.d.) model~\cite{aghazadeh_hierarchical_2025} and HMM for Claude. In Coding (red plots), the distributions nearly coincide, consistent with $K=1$ and approximately independent outcomes. In Reasoning (blue plots), the HIP yields a narrower distribution than the HMM. This reflects sequential dependence: success probabilities vary across latent states, and ignoring this causes HIP to treat correlated observations as independent, understating uncertainty. This effect propagates to the LLM level (green plots), where HIP yields a narrow distribution and HMM a broader, heavy-tailed one.

\begin{figure}
    \centering
    \includegraphics[width=0.7\linewidth]{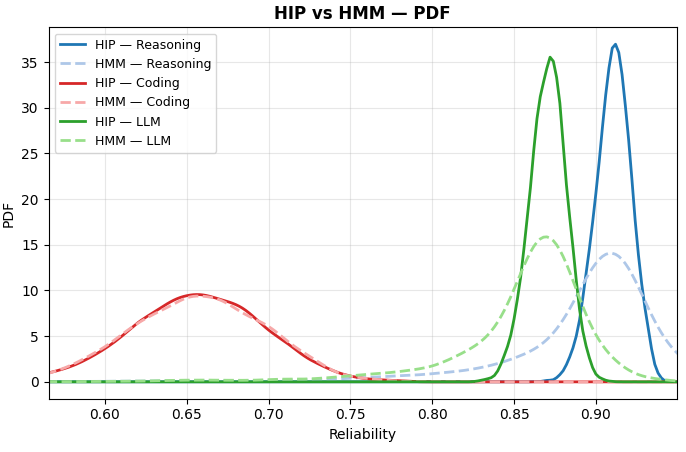}
    \caption{
    Posterior reliability PDFs for Claude under HIP and HMM. Differences are negligible in Coding ($K=1$) but substantial in Reasoning ($K=3$) and overall, with greater HMM uncertainty.}
    \label{fig_pdf_hmm_hip}
\end{figure}

\paragraph{Sensitivity to task ordering}
\del{Because benchmark datasets do not define a unique operational sequence, we assess the sensitivity of HMM-based reliability to task ordering using Claude-generated outcomes. We fix the observed pairs $(j_t, y_t)$ and vary only their order using three schemes: original (Sec.~\ref{sec_problem_setting}), full shuffle, and block shuffle (random permutation of contiguous blocks). For each sequence, we refit the HMM with fixed domain-specific $K$ ($K=3$ for Reasoning, $K=1$ for Coding).}

\del{We report posterior means and 95\% credible interval (CI) widths (Table~\ref{tab_ordering_sensitivity}). In coding, results are highly stable, with negligible changes across orderings, indicating weak sequential dependence. In contrast, reasoning shows greater sensitivity: while mean reliability remains similar, credible interval width varies across orderings, increasing under full shuffling and decreasing under block shuffling. This indicates that temporal structure influences inferred dependence, although the main conclusion remains that the HMM captures substantial uncertainty not explained by a single ordering.}

Because benchmark datasets do not define a unique operational sequence, we test ordering sensitivity using Claude-generated outcomes. We keep the pairs $(j_t,y_t)$ fixed and reorder them using three schemes: original (Sec.~\ref{sec_problem_setting}), full shuffle, and block shuffle. For each sequence, we refit the HMM with fixed domain-specific $K$ ($K=3$ for Reasoning, $K=1$ for Coding). Table~\ref{tab_ordering_sensitivity} reports means and 95\% CI widths. Coding is stable across orderings, indicating weak sequential dependence. Reasoning is more sensitive: mean reliability remains similar, but CI width increases under full shuffling and decreases under block shuffling. Thus, temporal structure affects inferred dependence, while the main conclusion remains unchanged.

\begin{table}[!htbp]
\centering
\caption{Sensitivity of posterior reliability to task ordering.}
\begin{tabular}{lccc}
\toprule
Domain & Ordering & Mean & CI width \\
\midrule
Coding & Original      & 0.6563 & 0.1608 \\
Coding & Full shuffle  & 0.6562 & 0.1625 \\
Coding & Block shuffle & 0.6559 & 0.1628 \\
\midrule
Reasoning & Original      & 0.8933 & 0.2182 \\
Reasoning & Full shuffle  & 0.8892 & 0.2279 \\
Reasoning & Block shuffle & 0.8945 & 0.2047 \\
\bottomrule
\end{tabular}
\label{tab_ordering_sensitivity}
\end{table}

\paragraph{Synthetic Validation}
\label{sec_synthetic_validation}

To provide evidence for the HMM beyond differences in posterior uncertainty, we conduct a controlled synthetic experiment with a known data-generating process. Data are generated under two settings: an HMM with a high self-transition probability (0.90) and an independent process. We compare a Beta-Binomial\footnote{It serves as a reference to show that the HMM's advantage over i.i.d.\ models is not solely due to emission flexibility.} baseline, a hierarchical i.i.d.\ model (HIP~\cite{aghazadeh_hierarchical_2025}), and the HMM (Reasoning: $K=3$, Coding: $K=1$ for Claude, Table~\ref{tab_logscores}).

Table~\ref{tab_synthetic_hmm_validation_fixed_delta} provides two findings supporting the proposed model. First, under the dependent process, the HMM achieves a better out-of-sample predictive log score than the HIP at the Reasoning domain level ($-716.73$ vs $-737.65$), indicating improved sequential prediction. It also achieves the smallest estimation error ($0.0006$ versus $0.0075$ for HIP and $0.0074$ for the Beta-Binomial baseline), because the strong sequential dependence makes the latent-state structure easier to learn. Second, under the independent process, the HMM and HIP achieve nearly identical log scores ($-718.23$ vs $-717.93$), while the HMM error increases. This is expected because, in the absence of sequential structure, the additional states introduce unnecessary variance, and the predictive log score correctly favors $K=1$, consistent with the real-data results.
\begin{table}[!htbp]
\centering
\caption{Validation under known processes.}
\label{tab_synthetic_hmm_validation_fixed_delta}
\begin{tabular}{llccc}
\hline
\textbf{Scenario} & \textbf{Model} &
\textbf{Error} &
\textbf{CI Width} &
\textbf{Log Score} \\
\hline
Dependent
& HMM & 0.0006 & 0.1621 & -716.73 \\
& HIP & 0.0075 & 0.0244 & -737.65 \\
& BB  & 0.0074 & 0.0237 & N/A \\
\hline
Independent
& HMM & 0.0217 & 0.1457 & -718.23 \\
& HIP & 0.0080 & 0.0257 & -717.93 \\
& BB  & 0.0081 & 0.0252 & N/A \\
\hline
\end{tabular}
\end{table}

\paragraph{Sensitivity to OPs}
\label{sec_op_sensitivity}


To assess sensitivity to the OP, we recomputed the HMM-based Reasoning-domain reliability for Claude under four alternative OPs while keeping the HMM posterior fixed and varying only the weights used to aggregate BoolQ and RACE-H reliabilities. The OPs comprise a Uniform OP, a Benchmark-frequency OP, and two deployment-skewed OPs favoring either BoolQ or RACE-H (90 \% vs 10\%).

Table~\ref{tab_op_sensitivity_claude_reasoning} shows that the posterior mean reliability is relatively stable across OPs (0.8905--0.9027), with the highest value under the BoolQ-skewed OP and the lowest under the RACE-H-skewed OP, consistent with BoolQ's slightly higher empirical accuracy. In contrast, uncertainty is more sensitive to the OP assumption: the credible-interval width increases from about 0.20 under the Uniform and Benchmark-frequency OPs to 0.2652 under the RACE-H-skewed OP. Thus, the assumed deployment distribution affects both the reliability estimate and the uncertainty of the resulting reliability claim.
\begin{table}[!htbp]
\centering
\caption{OP sensitivity of Claude Reasoning reliability.}
\label{tab_op_sensitivity_claude_reasoning}
\begin{tabular}{p{2.75cm}p{0.85cm}p{0.85cm}p{1cm}c}
\hline
\textbf{OP assumption} &
$\boldsymbol{\Omega_{\mathrm{BoolQ}}}$ &
$\boldsymbol{\Omega_{\mathrm{RACE}}}$ &
\textbf{Mean} &
\textbf{CrI width}\\
\hline
Uniform & 0.50 & 0.50 & 0.8966 & 0.2004 \\
Benchmark-frequency& 0.483 & 0.517 & 0.8963 & 0.2025\\
BoolQ-skewed & 0.90 & 0.10 & 0.9027 & 0.2072 \\
RACE-skewed& 0.10 & 0.90 & 0.8905 & 0.2652 \\
\hline
\end{tabular}
\end{table}

\paragraph{Interpretation of learned latent states}
\label{sec_interpretation_latent_state}
Following Remark~\ref{remark_interpreting_states}, we examine the learned latent states through their transition dynamics and state-dependent success probabilities. The states are associated with distinct observed success rates: 0.950 (State 1), 0.162 (State 2), and 0.660 (State 3). The estimated self-transition probabilities for the Claude Reasoning domain are $(0.9778, 0.9385, 0.9521)$ (matrix~A) for States 1--3, respectively, indicating that the inferred states tend to persist across consecutive tasks rather than changing rapidly.
\[
A=
\begin{bmatrix}
0.9778 & 0.0180 & 0.0042\\
0.0120 & 0.9385 & 0.0495\\
0.0300 & 0.0179 & 0.9521
\end{bmatrix}
\]

Although the latent states are not directly observable and have no predefined physical interpretation, they are data-driven statistical constructs inferred from the observed benchmark outcomes rather than internal LLM states. Since LLMs remain largely black boxes, a precise interpretation of these states is not currently feasible. Nevertheless, their distinct state-dependent success probabilities and persistence are consistent with the predictive selection of $K=3$ for the benchmark-derived sequential sessions.

\section{Discussion and conclusion}
\label{sec_discussion_conclusion}

\del{This work highlights the importance of accounting for temporal
dependence when assessing LLM reliability under realistic usage
conditions. While traditional approaches assume independence,
practical deployments involve continuous interactions in which
performance (success probabilities $\theta_{ij,s}$) may evolve
over time due to contextual memory and interaction dynamics.
As a result, reliability is not purely a static property, but
depends on how performance evolves across sequences.
The proposed HMM captures this evolution through a latent
interaction state inferred from observed outcomes. This state
can be interpreted as reflecting interaction conditions such
as context quality or reasoning stability, though this remains
a modeling assumption. Experimentally, the impact of accumulated
context on performance is evident, though its magnitude varies across domains.} 
This work demonstrates that accounting for temporal dependence can affect LLM reliability assessment. By extending a hierarchical Bayesian reliability framework with a Hidden Markov Model, we capture sequential dependence through a latent interaction state and obtain posterior reliability distributions at the subdomain, domain, and LLM levels. Across the benchmark-constructed interaction sessions considered in this study, the proposed approach consistently produces broader reliability posteriors when sequential dependence is present, indicating that models assuming independent task outcomes may underestimate uncertainty and yield overconfident reliability claims. The synthetic validation study further shows that the HMM provides improved predictive performance when the data-generating process contains persistent sequential structure, while reducing to behavior similar to the independent model when such structure is absent.

\del{As shown in experiments, different task orderings lead to 
measurable differences in posterior estimates and uncertainty, 
suggesting that standard benchmark protocols relying on fixed 
ordering and aggregated outcomes may overlook temporal dynamics 
and underestimate uncertainty.}



The experimental setup is intended as a proof-of-concept assessment of the proposed framework rather than a direct characterization of real-world LLM reliability. Although Table~IV shows that the main conclusions are robust across alternative task orderings, the benchmark-constructed sessions do not fully represent naturally occurring interactions, which may involve domain switching, user adaptation, and longer conversational dependencies. Therefore, these results should not be interpreted as direct evidence of real-world context accumulation. Future work should evaluate the framework on real interaction traces and deployment-representative OPs.

\del{The OP is treated as a static distribution over task types,
specifying relative frequencies but not the temporal structure
of task arrivals. In the current formulation, task types are
drawn independently at each step and independently of the
interaction state, separating task arrival from performance
dynamics and providing a tractable baseline for isolating
sequential dependence in model behavior.}
The OP specifies task frequencies, not task ordering. This allows the model to focus on dependence arising from the interaction dynamics rather than task arrivals.
In practice, however, task sequences are often structured.
Users may adapt their behavior in response to model
performance, creating a feedback loop between interaction
state and task selection that is not captured here.\del{To illustrate, consider a coding assistant session in which
70\% of tasks are code generation and 30\% are debugging.
The static OP correctly captures these marginal frequencies,
but treats tasks as independent draws. In reality, tasks may
follow a structured pattern (e.g., generate $\to$ debug),
and users may change behavior after failures.}
\del{This has implications for reliability. If debugging tasks are
harder and failures degrade subsequent performance through
accumulated context, then reliability under such structured
sequences may be lower than predicted by the static OP, which
captures only marginal frequencies and not temporal patterns.}For example, a coding assistant session may consist of 70\% code-generation tasks and 30\% debugging tasks. The OP captures these frequencies but not the ordering of tasks, such as generate $\rightarrow$ debug sequences. This can affect reliability. If debugging tasks are harder and performance depends on previous interactions, reliability under such sequences may be lower than predicted by a static OP, which ignores task ordering.
\del{In addition, the use of a population-level OP implies that reported reliability reflects average performance across heterogeneous users and interaction patterns, potentially masking significant variation at the individual user level. A distinction between population-level and instance-specific OPs has been explored in autonomous vehicle safety, where fleet-level averages were shown to obscure safety-critical differences across individual instances~\cite{popov_dynamic_2025}. An analogous concern applies here: users with distinct task distributions may experience reliability that differs considerably from the population-level estimate.} In addition, a population-level OP implies that reliability reflects average performance across heterogeneous users and interaction patterns, potentially masking substantial variation at the individual-user level. A similar issue has been identified in autonomous vehicle safety, where fleet-level averages were shown to obscure safety-critical differences between individual instances~\cite{popov_dynamic_2025}. Analogously, users with different task distributions may experience reliability that differs substantially from the population-level estimate.
\del{Extending the framework to instance-specific OPs would enable more personalized and deployment-relevant reliability assessment.} Instance-specific OPs could enable more personalized and deployment-relevant reliability assessment.

\del{Future work will extend the framework to capture more complex interaction dynamics, including cross-domain dependence via a shared global latent state, time-varying transitions, and adaptive task arrival processes. Evaluating the proposed approach across a broader range of LLMs is also important, as sequential dependence may vary with model architecture and context handling. Evaluating the proposed approach on task sequences collected from real user interactions, rather than artificially constructed orderings, would provide a more grounded assessment of how sequential dependence manifests in practice.} 

Future work will investigate deployment-derived and instance-specific OPs, real user interaction traces, cross-domain dependence, time-varying transition dynamics, adaptive task-arrival processes, and a broader range of LLMs.

\del{Overall, these results suggest that sequential dependence is a practically relevant and measurable property of LLM reliability, and that explicitly modeling interaction dynamics is essential for deployment-relevant evaluation.}

\section*{Acknowledgment}
This work receives funding from the European Union's EU Framework Program for Research and Innovation Europe Horizon (grant agreement No 101202457). 
Views expressed are those of the authors only and do not necessarily reflect those of the European Union or European Research Executive Agency (REA). Neither the European Union nor the granting authority can be held responsible for them.
SK's contribution is supported by the UKRI Future Leaders Fellowship Grant [MR/S035176/1].

\bibliography{ref}  

@article{kaufmann_regression_1987,
  title={Regression models for nonstationary categorical time series: asymptotic estimation theory},
  author={Kaufmann, Heinz},
  journal={The Annals of Statistics},
  pages={79--98},
  year={1987},
  publisher={JSTOR}
}

@article{barry_autoreg_1993,
  title={A Bayesian analysis for change point problems},
  author={Barry, Daniel and Hartigan, John A},
  journal={Journal of the American Statistical Association},
  volume={88},
  number={421},
  pages={309--319},
  year={1993},
  publisher={Taylor \& Francis}
}

@inproceedings{liu_uncertainty_2025,
  title={Uncertainty quantification and confidence calibration in large language models: A survey},
  author={Liu, Xiaoou and Chen, Tiejin and Da, Longchao and Chen, Chacha and Lin, Zhen and Wei, Hua},
  booktitle={Proceedings of the 31st ACM SIGKDD Conference on Knowledge Discovery and Data Mining V. 2},
  pages={6107--6117},
  year={2025}
}

@inproceedings{geng_survey_2024,
  title={A survey of confidence estimation and calibration in large language models},
  author={Geng, Jiahui and Cai, Fengyu and Wang, Yuxia and Koeppl, Heinz and Nakov, Preslav and Gurevych, Iryna},
  booktitle={Proceedings of the 2024 Conference of the North American Chapter of the Association for Computational Linguistics: Human Language Technologies (Volume 1: Long Papers)},
  pages={6577--6595},
  year={2024}
}

@article{dai_pre_2025,
  title={Pre-trained Large Language Models Learn Hidden Markov Models In-context},
  author={Dai, Yijia and Gao, Zhaolin and Sattar, Yahya and Dean, Sarah and Sun, Jennifer J},
  journal={arXiv preprint arXiv:2506.07298},
  year={2025}
}

@article{popov_dynamic_2025,
  title={Dynamic safety assessment of autonomous vehicle based on multivariate Bayesian inference (DyAVSA)},
  author={Popov, Peter},
  journal={Journal of Reliable Intelligent Environments},
  volume={11},
  number={3},
  pages={14},
  year={2025},
  publisher={Springer}
}

@inproceedings{bishop_pods_1988,
  title={PODS revisited-a study of software failure behaviour},
  author={Bishop, Peter G and Pullen, FD},
  booktitle={[1988] The Eighteenth International Symposium on Fault-Tolerant Computing. Digest of Papers},
  pages={2--8},
  year={1988},
  organization={IEEE}
}

@book{cappe_inference_2005,
  title={Inference in hidden Markov models},
  author={Capp{\'e}, Olivier and Moulines, Eric and Ryd{\'e}n, Tobias},
  year={2005},
  publisher={Springer}
}

@article{lim_evaluating_2025,
  title={Evaluating motivational interview quality using large language models and hidden Markov models},
  author={Lim, Kyungho and Jung, Young-Chul and Kim, Byung-Hoon},
  journal={BMC psychiatry},
  volume={25},
  number={1},
  pages={908},
  year={2025},
  publisher={Springer}
}

@article{ildiz_self_2024,
  title={From self-attention to markov models: Unveiling the dynamics of generative transformers},
  author={Ildiz, M Emrullah and Huang, Yixiao and Li, Yingcong and Rawat, Ankit Singh and Oymak, Samet},
  journal={arXiv preprint arXiv:2402.13512},
  year={2024}
}

@article{zekri_large_2024,
  title={Large language models as markov chains},
  author={Zekri, Oussama and Odonnat, Ambroise and Benechehab, Abdelhakim and Bleistein, Linus and Boull{\'e}, Nicolas and Redko, Ievgen},
  journal={arXiv preprint arXiv:2410.02724},
  year={2024}
}

@article{rabiner_tutorial_2002,
  title={A tutorial on hidden Markov models and selected applications in speech recognition},
  author={Rabiner, Lawrence R},
  journal={Proceedings of the IEEE},
  volume={77},
  number={2},
  pages={257--286},
  year={2002},
  publisher={Ieee}
}

@article{liu_resilience_2025,
  title={Resilience evaluation of multi-feature system based on hidden Markov model},
  author={Liu, Jiaying and Zhang, Jun and Tian, Qingfeng and Wu, Bei},
  journal={Reliability Engineering \& System Safety},
  volume={253},
  pages={110561},
  year={2025},
  publisher={Elsevier}
}

@article{gamiz_hidden_2023,
  title={Hidden Markov models in reliability and maintenance},
  author={G{\'a}miz, Mar{\'\i}a Luz and Limnios, Nikolaos and del Carmen Segovia-Garc{\'\i}a, Mar{\'\i}a},
  journal={European Journal of Operational Research},
  volume={304},
  number={3},
  pages={1242--1255},
  year={2023},
  publisher={Elsevier}
}

@article{gneiting_strictly_2007,
  title={Strictly proper scoring rules, prediction, and estimation},
  author={Gneiting, Tilmann and Raftery, Adrian E},
  journal={Journal of the American statistical Association},
  volume={102},
  number={477},
  pages={359--378},
  year={2007},
  publisher={Taylor \& Francis}
}

@article{aghazadeh_hierarchical_2025,
  title={A Hierarchical Imprecise Probability Approach to Reliability Assessment of Large Language Models},
  author={Aghazadeh-Chakherlou, Robab and Guo, Qing and Khastgir, Siddartha and Popov, Peter and Zhang, Xiaoge and Zhao, Xingyu},
  journal={Reliabilty Engineering and System Safety},
  year={2026}
}

@inproceedings{bishop_variation_1993,
  title={The variation of software survival time for different operational input profiles (or why you can wait a long time for a big bug to fail)},
  author={Bishop, Peter G},
  booktitle={FTCS-23 The Twenty-Third International Symposium on Fault-Tolerant Computing},
  pages={98--107},
  year={1993},
  organization={IEEE}
}

@article{strigini_testing_1996,
  title={On testing process control software for reliability assessment: the effects of correlation between successive failures},
  author={Strigini, Lorenzo},
  journal={Software Testing, Verification and Reliability},
  volume={6},
  number={1},
  pages={33--48},
  year={1996},
  publisher={Wiley Online Library}
}

@article{littlewood_conservative_1997,
	title = {Some conservative stopping rules for the operational testing of safety critical software},
	volume = {23},
	number = {11},
	journal = {IEEE Trans. on software Engineering},
	author = {Littlewood, Bev and Wright, David},
	year = {1997},
	pages = {673--683},
}

@article{goseva_popstojanova_failure_2000,
	title = {Failure correlation in software reliability models},
	volume = {49},
	number = {1},
	journal = {IEEE Trans. on Reliability},
	author = {Goseva-Popstojanova, K. and Trivedi, K. S.},
	year = {2000},
	pages = {37--48}
}

@article{chen_binary_1996,
	title = {A binary {Markov} process model for random testing},
	volume = {22},
	number = {3},
	journal = {IEEE Trans. on Softw. Eng.},
	author = {Chen, Sanping and Mills, Shirley},
	year = {1996},
	pages = {218--223},
}

@article{luettgau_hibayes_2025,
  title={HiBayES: A {Hierarchical Bayesian Modeling Framework for AI Evaluation Statistics}},
  author={Luettgau, Lennart and Coppock, Harry and Dubois, Magda and Summerfield, Christopher and Ududec, Cozmin},
  journal={arXiv preprint arXiv:2505.05602},
  year={2025}
}

@article{miller_adding_2024,
  title={Adding error bars to evals: A statistical approach to language model evaluations},
  author={Miller, Evan},
  journal={arXiv preprint arXiv:2411.00640},
  year={2024}
}

@article{salako_unnecessity_2023,
  title={The unnecessity of assuming statistically independent tests in bayesian software reliability assessments},
  author={Salako, Kizito and Zhao, Xingyu},
  journal={IEEE Trans. on Software Engineering},
  volume={49},
  number={4},
  pages={2829--2838},
  year={2023},
  publisher={IEEE}
}
\bibliographystyle{IEEEtran}

\appendices

\section{Outline of Proofs of Theorems~\ref{thm_postrior_for_HMM} and~\ref{thm_posterior_reliability}}
\label{sec_proof_theorem}

\begin{proof}
\del{The proof outline is organized in five steps. 
First, Steps 1--3 establish how the LLM behavior is inferred from data (Theorem~\ref{thm_postrior_for_HMM}). Then, Steps 4, 5 use this result to derive the reliability of the LLM for future tasks (Theorem~\ref{thm_posterior_reliability}).} The proof proceeds in five steps. Steps 1--3 derive the posterior
distribution of the HMM parameters (Theorem~\ref{thm_postrior_for_HMM}).
Steps 4--5 propagate this posterior to the reliability quantities
and verify the mixed-subdomain recursion
(Theorem~\ref{thm_posterior_reliability}).

\textbf{Step 1: Joint distribution factorization.}
Condition on the observed subdomain labels $j_{1:T}$. Let
$z_{1:T}\in\mathcal{Z}^T$ and $y_{1:T}\in\{0,1\}^T$. By the model
assumptions (initial distribution, Markov property, and conditional
independence of emissions given $(Z_{i,t},j_t)$), the joint
distribution factorizes as
\begin{align}
& p(Z_{i,1:T}=z_{1:T},\,Y_{i,1:T}=y_{1:T}\mid j_{1:T},\theta_i,\delta_i,A_i)
\nonumber\\
&\hspace{0.2cm}
= p(Z_{i,1}=z_1\mid \delta_i)
\prod_{t=1}^{T-1} p(Z_{i,t+1}=z_{t+1}\mid Z_{i,t}=z_t,A_i)
\nonumber\\
&\hspace{0.45cm}\times
\prod_{t=1}^{T} p(Y_{i,t}=y_t\mid Z_{i,t}=z_t,j_t,\theta_i) \nonumber
\end{align}
Using $p(Z_{i,1}=s\mid\delta_i)=(\delta_i)_s$, $p(Z_{i,t+1}=s'\mid Z_{i,t}=s,A_i)=(A_i)_{s,s'}$, and $p(Y_{i,t}=y\mid Z_{i,t}=s,j_t=j,\theta_i)=
    \theta_{ij,s}^{\,y}(1-\theta_{ij,s})^{\,1-y}$,
provides the standard HMM joint distribution with
context-dependent emissions which is the standard HMM factorization with context-dependent emissions.
\del{This step translates our verbal model description into a mathematical form. We first choose the starting hidden state,
then move between states,
and at each time step generate an outcome from the current state.}

\textbf{Step 2: Likelihood via marginalization over latent states.}
By definition, the likelihood of $(\theta_i,A_i)$ given the observed
outcomes $y_{1:T}$ and conditioned on $j_{1:T}$ is
\[
\mathcal{L}_i(\theta_i,A_i)
=
p(Y_{i,1:T}=y_{1:T}\mid j_{1:T},\theta_i,\delta_i,A_i).
\]
Using the law of total
probability:
\begin{align}
&\mathcal{L}_i(\theta_i,A_i)
= \nonumber \\
& \sum_{z_{1:T}\in\mathcal{Z}^T}
p(Z_{i,1:T}=z_{1:T},\,Y_{i,1:T}=y_{1:T}
\mid j_{1:T},\theta_i,\delta_i,A_i) \nonumber
\label{eq_marginal_proof}
\end{align}
\del{In this step we want the likelihood of the observed data only.
But the hidden states $Z_{i,1:T}$ are not observed. Thus, we sum over all possible hidden state sequences. That is, we do not know which hidden state path actually happened, so we consider every possible hidden state path,
compute the probability of the data under each path,
and add them all together.
That gives the probability of the observed data, regardless of which hidden path produced it.}
Since the latent state sequence is unobserved, the likelihood is obtained by marginalizing over all possible state sequences.

\textbf{Step 3: Posterior characterization (Theorem~\ref{thm_postrior_for_HMM}).}
\del{Let $\psi_i = \{\theta_{ij,s}, A_i, \mu_i, \nu_i\}$ denote the model parameters, where $(\mu_i,\nu_i)$ are domain-level hyperparameters governing the hierarchical prior on $\theta_{ij,s}$. Given the observed data $(j_{1:T}, Y_{i,1:T})$, and the prior factorization} Applying Bayes' theorem with the likelihood from Step 2
and the hierarchical prior factorization yields
\[
p(\psi_i)=p(\theta_i\mid\mu_i,\nu_i)\,p(\mu_i,\nu_i)\,p(A_i),
\]
Using the Bayes' rule:
\begin{align}
& p(\psi_i \mid Y_{i,1:T}, j_{1:T})
= \nonumber \\
& \hspace{1cm}\frac{
\mathcal{L}_i(\theta_i,A_i)\,
p(\theta_i \mid \mu_i,\nu_i)\,
p(\mu_i,\nu_i)\,
p(A_i)
}{
\int
\mathcal{L}_i(\theta_i,A_i)\,
p(\theta_i \mid \mu_i,\nu_i)\,
p(\mu_i,\nu_i)\,
p(A_i)
\, d\psi_i
} \nonumber
\end{align}

\noindent\textbf{Step 4: Posterior distributions of reliability
(Theorem~\ref{thm_posterior_reliability}).}
The reliability quantities in
Theorem~\ref{thm_posterior_reliability}
are measurable functions of the model parameters.
The subdomain reliability $R_{ij}(n;\psi_i)$ is obtained
by averaging the state-conditional reliability
$R_{ij,s}(n;\psi_i)$ over the predictive state distribution
$\rho_i(s;\psi_i)$. Here $\rho_i(s;\psi_i)$ denotes the distribution of the latent state governing
the first future task: for a continuing session,
\[
\rho_i(s';\psi_i)
= \sum_{s=1}^K p(Z_{i,T}=s \mid Y_{i,1:T}, j_{1:T}, \psi_i)\,(A_i)_{s,s'},
\]
obtained by propagating the filtered state distribution at the final observed
step through one transition; for a fresh session, $\rho_i(s;\psi_i)=\delta_{i,s}$.
The domain-level reliability $R_i(n;\psi_i)$ is obtained
by applying the $n$-step mixed-subdomain recursion to the
state-conditional average success probability
$\bar{\theta}_{i,s} = \sum_{j=1}^{n_i}\Omega_{ij}\theta_{ij,s}$
and weighting by 
$\rho_i(s';\psi_i)$:
\[
R_i(n;\psi_i) = \sum_{s=1}^{K}\rho_i(s;\psi_i)\,
\bar{R}_{i,s}(n;\psi_i).
\]
The overall LLM reliability $R_L(n;\psi)$ aggregates
domain-level reliabilities using weights $W_i$:
\[
R_L(n;\psi) = \sum_{i=1}^{m} W_i\, R_i(n;\psi_i).
\]

\del{It is important to note that, although each quantity
$R_{ij}(n;\psi_i)$, $R_i(n;\psi_i)$, and $R_L(n;\psi)$
represents a probability of success for future tasks,
these quantities depend on the model parameters $\psi_i$,
which are themselves uncertain. Under the Bayesian
framework, $\psi_i$ is treated as a random variable with
posterior distribution $p(\psi_i \mid Y_{i,1:T}, j_{1:T})$.
Consequently, the reliability quantities are also random
variables, since they are obtained by evaluating
deterministic functions at the random parameter vector.
Therefore, it is meaningful to consider probability
statements about these reliability values, such as the
probability that the reliability lies within a given range.}

\del{Since Theorem~\ref{thm_postrior_for_HMM} establishes that
$p(\psi_i \mid Y_{i,1:T}, j_{1:T})$ is well defined, the
posterior distributions of these reliability quantities
are the pushforward measures induced by this posterior.
Hence, for any measurable set $B\subseteq[0,1]$,}

Since Theorem~\ref{thm_postrior_for_HMM} establishes the posterior
distribution
$p(\psi_i \mid Y_{i,1:T}, j_{1:T})$, and the reliability
quantities are measurable functions of $\psi_i$, their posterior
are the corresponding pushforward measures. Hence, for
any measurable set $B\subseteq[0,1]$,
\begin{align}
   & p \big(R_{ij}(n;\psi_i)\in B
     \mid Y_{i,1:T}, j_{1:T}\big) = \nonumber \\
   & \hspace{1cm}
     \int \mathbf{1}\{R_{ij}(n;\psi_i)\in B\}\,
     p(\psi_i \mid Y_{i,1:T}, j_{1:T})\,d\psi_i,
     \nonumber\\
   & p \big(R_i(n;\psi_i)\in B
     \mid Y_{i,1:T}, j_{1:T}\big) = \nonumber\\
   & \hspace{1cm}
     \int \mathbf{1}\{R_i(n;\psi_i)\in B\}\,
     p(\psi_i \mid Y_{i,1:T}, j_{1:T})\,d\psi_i,
     \nonumber\\
   & p \big(R_L(n;\psi)\in B
     \mid \{Y_{i,1:T},j_{1:T}\}_{i=1}^m\big) = \nonumber \\
   & \hspace{1cm}
     \int \mathbf{1}\{R_L(n;\psi)\in B\}\,
     p(\psi \mid \{Y_{i,1:T},j_{1:T}\}_{i=1}^m)\,d\psi
     \nonumber
\end{align}
where $\mathbf{1}\{\cdot\}$ denotes the indicator function.
Thus, the reliability quantities admit well-defined
posterior distributions.

\textbf{Step 5: Verification of the recursive representation
for $\bar{R}_{i,s}(n;\psi_i)$.}
 
\del{Here $x_\tau$ denotes the $\tau$-th future task and
$I(x_\tau)\in\{0,1\}$ its success indicator.} denote the $\tau-th$ future task and $I(x_\tau)$ its success indicator. We verify by
induction that
\begin{equation}
  \bar{R}_{i,s}(n;\psi_i)
  =
  p\!\left(
    \bigcap_{\tau=1}^{n}\{I(x_\tau)=1\}
    \;\Big|\;
    Z_1 = s
  \right),
  \label{eq_step5_claim}
\end{equation}
\del{where the future tasks $x_1,\dots,x_n$ are drawn from domain
$D_i$ with subdomain labels $j_\tau \sim \Omega_{i\cdot}$
drawn independently at each step, independently of the latent
state $Z_\tau$.} where future subdomain labels are drawn independently from $\Omega_i$· and independently of the latent state process.
 
\emph{Base case ($n=1$).}  Since $j_1$ is drawn from
$\Omega_{i\cdot}$ independently of $Z_1$,
\begin{align}
    & p(I(x_1)=1 \mid Z_1=s)
  = \nonumber \\
  & \sum_{j=1}^{n_i}
  \Omega_{ij}\,
  p(I(x_1)=1 \mid Z_1=s,\, j_1=j)
  = \nonumber \\
  & \sum_{j=1}^{n_i} \Omega_{ij}\,\theta_{ij,s}
  =
  \bar{\theta}_{i,s}
  =
  \bar{R}_{i,s}(1;\psi_i) \nonumber
\end{align}

\emph{Inductive step.} \del{Assume \eqref{eq_step5_claim} holds
for $n-1$. Under the mixed-subdomain sampling scheme described
above, and using the law of total probability, the Markov
property of $\{Z_\tau\}$, the conditional independence of the
current outcome $I(x_1)$ from future latent states and future outcomes given $(Z_1,j_1)$, and the independence of $j_1$
from $Z_1$, we obtain}
Assume \eqref{eq_step5_claim} holds for $n-1$. Using the mixed-subdomain sampling scheme, the law of total probability, the Markov property of $\{Z_\tau\}$, the conditional independence of $I(x_1)$ from future states and outcomes given $(Z_1,j_1)$, and the independence of $j_1$ and $Z_1$, we obtain
\begin{align}
  &p\!\left(
    \bigcap_{\tau=1}^{n}\{I(x_\tau)=1\}
    \;\Big|\; Z_1=s
  \right)
  = \nonumber \\
  & \sum_{j=1}^{n_i}
  \Omega_{ij}\,
  p(I(x_1)=1 \mid Z_1=s,\, j_1=j) \times
  \notag\\
  &
  \sum_{s'=1}^{K}
  p(Z_2=s' \mid Z_1=s)
  \cdot
  p\!\left(
    \bigcap_{\tau=2}^{n}\{I(x_\tau)=1\}
    \;\Big|\; Z_2=s'
  \right)
  \notag\\
  &=
  \sum_{j=1}^{n_i} \Omega_{ij}\,\theta_{ij,s}
  \;\cdot\;
  \sum_{s'=1}^{K}
  (A_i)_{s,s'}\,
  \bar{R}_{i,s'}(n-1;\psi_i)
  \notag \\
  &=
  \bar{\theta}_{i,s}
  \sum_{s'=1}^{K}
  (A_i)_{s,s'}\,
  \bar{R}_{i,s'}(n-1;\psi_i)=\bar{R}_{i,s}(n;\psi_i), \nonumber
\end{align}

where the third line applies the induction hypothesis to each
$\bar{R}_{i,s'}(n-1;\psi_i)$ and collects
$\sum_j \Omega_{ij}\,\theta_{ij,s} = \bar{\theta}_{i,s}$.
 
\del{This completes the induction.  The measurability argument
(that $\bar{R}_{i,s}(n;\psi_i)$ is a measurable function of
$\psi_i$, and hence admits a posterior distribution) follows
by the same argument as for $R_{ij,s}(n;\psi_i)$ in the
previous version, since $\bar{\theta}_{i,s}$ is a finite
linear combination of the measurable functions $\theta_{ij,s}$.}
This establishes the recursion. Measurability follows because
$\bar{\theta}_{i,s}$ is a finite linear combination of the measurable
parameters $\theta_{ij,s}$.
\del{The conditional subdomain recursion $R_{ij,s}(n;\psi_i)$
verified in the earlier proof remains valid as a diagnostic
quantity: it equals
$p \bigl(\bigcap_{\tau=1}^{n}\{I(x_\tau)=1\} \mid Z_1=s,\,
x_1,\dots,x_n\in S_{ij}\bigr)$,
i.e., the probability of $n$ consecutive successes conditional
on all future tasks belonging to subdomain $S_{ij}$.}
The corresponding subdomain recursion $R_{ij,s}(n;\psi_i)$
follows analogously and retains its interpretation as the
probability of n consecutive successes conditional on all
future tasks belonging to $S_{ij}$.

Finally, the posterior in Theorem~\ref{thm_postrior_for_HMM} is approximated via Gibbs sampling and forward–backward procedures, and reliability posteriors are obtained by evaluating these functionals at the posterior draws. A complete proof is available in GitHub repository: \url{https://github.com/llmReliability/hmm-llm-reliability-proof}.

\end{proof}

\end{document}